\pdfoutput=1
\documentclass[11pt]{article}

\usepackage[letterpaper,margin=1in]{geometry}

\usepackage{lipsum}

\newcommand\blfootnotea[1]{\begingroup
  \renewcommand\thefootnote{}\footnote{#1}\endgroup
}

\usepackage[utf8]{inputenc} \usepackage[T1]{fontenc}    

\usepackage{url}            \usepackage{booktabs}       \usepackage{nicefrac}       \usepackage{microtype}      \usepackage{enumitem}
\usepackage{dsfont}
\usepackage{multicol}
\usepackage{xcolor}
\usepackage{mdframed}

\usepackage{amsthm,amsfonts,amsmath,amssymb,epsfig,color,float,graphicx,verbatim, enumitem,empheq,bm}

\usepackage{algorithmicx}
\usepackage[noend]{algpseudocode}
\usepackage{algorithm}

\usepackage{mathtools}
\usepackage{bbm}

\usepackage{thm-restate}

\definecolor{green}{rgb}{0.0, 0.5, 0.0}
\usepackage[colorlinks,citecolor=blue,linkcolor=magenta,bookmarks=true,hypertexnames=false]{hyperref}
\usepackage[nameinlink,capitalize]{cleveref}

\crefname{lemma}{lemma}{lemmata}
\crefname{claim}{claim}{claims}
\crefname{theorem}{theorem}{theorems}
\crefname{proposition}{proposition}{propositions}
\crefname{corollary}{corollary}{corollaries}
\crefname{claim}{claim}{claims}
\crefname{remark}{remark}{remarks}
\crefname{definition}{definition}{definitions}
\crefname{fact}{fact}{facts}
\crefname{question}{question}{questions}
\crefname{condition}{condition}{conditions}
\crefname{problem}{problem}{problems}
\crefname{algorithm}{algorithm}{algorithms}
\crefname{assumption}{assumption}{assumptions}
\crefname{notation}{notation}{notation}
\crefname{cond}{Condition}{Conditions}
\crefname{ineq}{Inequality}{Inequality}
\creflabelformat{ineq}{#2{\upshape(#1)}#3}
\crefname{sub}{Subsection}{Subsection}
\creflabelformat{Subsection}{#2{\upshape(#1)}#3}
\crefname{sdp}{SDP}{SDP}
\creflabelformat{sdp}{#2{\upshape(#1)}#3}
\crefname{lp}{LP}{LP}
\creflabelformat{lp}{#2{\upshape(#1)}#3}
\renewcommand\vec[1]{\mathbf{#1}}
\newcommand{\x}{\vec x}

\def\d{\mathrm{d}}

\newtheorem{theorem}{Theorem}[section]
\newtheorem{lemma}[theorem]{Lemma}
\newtheorem{proposition}[theorem]{Proposition}
\newtheorem{corollary}[theorem]{Corollary}
\newtheorem{claim}[theorem]{Claim}

\newtheorem{definition}[theorem]{Definition}

\newtheorem{fact}[theorem]{Fact}

\theoremstyle{definition}

\newtheorem{problem}[theorem]{Problem}

\renewcommand{\epsilon}{\varepsilon}
\newcommand{\eps}{\epsilon}

\newcommand{\poly}{\mathrm{poly}}
\newcommand{\polylog}{\mathrm{polylog}}
\newcommand{\dtv}{\mathrm{d}_\mathrm{TV}}

\def\D{\mathcal D}

\def\R{\mathbb R}

\def\N{\mathbb N}

\def\Z{\mathbb Z}

\newcommand{\cA}{\mathcal{A}}
\newcommand{\cB}{\mathcal{B}}
\newcommand{\cC}{\mathcal{C}}

\newcommand{\cI}{\mathcal{I}}

\newcommand{\cN}{\mathcal{N}}

\newcommand{\cP}{\mathcal{P}}

\newcommand{\cS}{\mathcal{S}}

\newcommand{\cU}{\mathcal{U}}

\newcommand{\cX}{\mathcal{X}}

\newcommand{\bA}{\vec{A}}
\newcommand{\bB}{\vec{B}}

\newcommand{\bI}{\vec{I}}

\newcommand{\bR}{\vec{R}}

\newcommand{\bU}{\vec{U}}
\newcommand{\bV}{\vec{V}}

\newcommand{\bx}{\mathbf{x}}
\newcommand{\by}{\mathbf{y}}
\newcommand{\bv}{\mathbf{v}}
\newcommand{\bu}{\mathbf{u}}
\newcommand{\bz}{\mathbf{z}}
\newcommand{\bw}{\mathbf{w}}

\DeclareMathOperator*{\pr}{\mathbf{Pr}}
\DeclareMathOperator*{\E}{\mathbf{E}}
\DeclareMathOperator*{\Var}{\mathbf{Var}}

\newcommand{\eqdef}{\stackrel{{\mathrm {\footnotesize def}}}{=}}

\newcommand{\op}{\textnormal{op}}
\newcommand{\fr}{\textnormal{F}}

\newcommand{\supp}{\mathrm{supp}}

\newcommand{\normal}{\mathcal{N}}
\def\d{\mathrm{d}}

\let\vec\mathbf

\def\colorful{0}

\ifnum\colorful=1
\newcommand{\new}[1]{{\color{red} #1}}
\newcommand{\blue}[1]{{\color{blue} #1}}

\else
\newcommand{\new}[1]{{#1}}
\newcommand{\blue}[1]{{#1}}

\fi

\allowdisplaybreaks

\title{SQ Lower Bounds for Learning Bounded Covariance GMMs\blfootnotea{Authors are in alphabetical order.}}

\author{
Ilias Diakonikolas\thanks{Supported by NSF Medium Award CCF-2107079, NSF Award CCF-1652862 (CAREER), 
and a DARPA Learning with Less Labels (LwLL) grant.
}\\
University of Wisconsin-Madison\\
{\tt ilias@cs.wisc.edu}\\
\and
Daniel M. Kane\thanks{Supported by NSF Medium Award CCF-2107547 and NSF Award CCF-1553288 (CAREER).}\\
University of California, San Diego\\
{\tt dakane@cs.ucsd.edu}
\and
Thanasis Pittas\thanks{Supported by NSF Medium Award CCF-2107079.}\\
University of Wisconsin-Madison\\
{\tt pittas@wisc.edu}\\
\and
Nikos Zarifis\thanks{Supported in part by DARPA Learning with Less Labels (LwLL) grant and NSF Award DMS-2023239 (TRIPODS).}\\
University of Wisconsin-Madison\\
{\tt zarifis@wisc.edu}\\
}

\begin{document}

\maketitle

\begin{abstract}
We study the complexity of learning 
mixtures of separated Gaussians with common unknown bounded
covariance matrix. Specifically, we focus on learning Gaussian mixture models (GMMs) on $\R^d$ 
of the form $P= \sum_{i=1}^k w_i \cN(\bm \mu_i,\vec \Sigma_i)$, 
where $\vec \Sigma_i = \vec \Sigma \preceq \vec I$
and $\min_{i \neq j} \|\bm \mu_i - \bm \mu_j\|_2 \geq k^\eps$ for some $\eps>0$. 
Known learning algorithms for this family of GMMs have complexity $(dk)^{O(1/\eps)}$. In this work, we
prove that any Statistical Query (SQ) algorithm for this problem requires complexity at least 
 $d^{\Omega(1/\eps)}$. \new{In the special case where the separation is on the order of $k^{1/2}$, 
 we additionally obtain fine-grained SQ lower bounds with the correct exponent.}
Our SQ lower bounds imply 
similar lower bounds for low-degree polynomial tests. 
\new{Conceptually,} our results provide evidence that known algorithms for 
this problem are nearly best possible. 
\end{abstract}

\setcounter{page}{0}

\thispagestyle{empty}

\newpage

\section{Introduction} \label{sec:intro}

We study the classical problem of learning Gaussian mixture models (GMMs) in high dimensions.
This problem has a long history, starting with the early work of \new{Pearson}~\cite{Pearson:94}
who introduced the method of moments in this context.
Over the past \new{three} decades, there has been a vast literature on learning GMMs 
in both statistics and theoretical computer science~\cite{Dasgupta:99, AroraKannan:01, VempalaWang:02,
AchlioptasMcSherry:05, FOS:06, KSV08, BV:08, MoitraValiant:10, BelkinSinha:10, SOAJ14, DK14, HardtP15, 
DHKK20, BakshiDHKKK20, DKKLT21, LM20-gmm, BD+20-gmm}. 
Here we focus on computational aspects \new{of this problem} 
with a focus on {\em information-computation tradeoffs} 
in high dimensions.

The learning setup is as follows: 
We have access to i.i.d.\ samples from an unknown $k$-GMM on $\R^d$, 
$P = \sum_{i=1}^k w_i \cN(\bm \mu_i, \vec \Sigma_i)$, where $w_i \geq 0$ are the mixing weights satisfying $\sum_{i=1}^k w_i = 1$, 
$\bm \mu_i \in \R^d$ are the unknown component means 
and $\vec \Sigma_i$ are the unknown component covariances. 
Roughly speaking, there are two versions of the learning problem:
(1) density estimation, where the goal is to compute a hypothesis distribution $H$ 
that is close to $P$ in total variation distance, and 
(2) parameter estimation\footnote{
A related task is that of clustering the sample 
based on the generating component. Once we have an accurate clustering, 
assuming one exists, we can individually learn the individual component parameters.}, 
where the goal is to approximate the target parameters $w_i, \bm \mu_i, \vec \Sigma_i$ within small error.
While density estimation of $k$-GMMs on $\R^d$ 
is information-theoretically solvable with $\poly(d, k)$ samples,
parameter estimation may require $2^{\Omega(k)}$ samples (even in one dimension)
if the individual components are close to each other~\cite{MoitraValiant:10}. 
On the other hand, under the standard separation assumption
that the components are ``nearly non-overlapping'', 
parameter estimation can also be solved with $\poly(d, k)$ samples. 
Here we focus on families of instances satisfying appropriate separation assumptions.
Even though such instances can be learned with $\poly(d, k)$ samples, it is by no means
clear that a $\poly(d, k)$-{\em time} learning algorithm exists.
In other words, we explore the relevant {\em information-computation tradeoffs} --- 
inherent tradeoffs between the sample complexity and the computational complexity of learning. 

A number of recent works have established information-computation tradeoffs in the context of 
learning GMMs. The first such result was given in~\cite{DKS17-sq} and applied to the class of 
Statistical Query (SQ) algorithms\footnote{Via a recent reduction~\cite{BBHLS20}, these SQ lower bounds imply
qualitatively similar low-degree testing lower bounds.}.
\new{Specifically,} \cite{DKS17-sq} constructed a hard family 
of GMMs (henceforth informally termed as ``parallel pancakes'') 
and showed that any SQ learner for this family requires super-polynomial time. 
\new{Interestingly, the} class of parallel pancakes is learnable with $O(k \log d)$ samples, 
while any SQ learning algorithm requires $d^{\Omega(k)}$ time. It is worth noting that 
subsequent work~\cite{BRST21, GupteVV22} established computational hardness 
for essentially the same class of instances, under widely-believed cryptographic assumptions.

In this work, we focus on a \new{simpler and well-studied} 
family of GMMs for which significantly faster learning algorithms are known. 
(We provide a detailed comparison between the family of instances \new{we consider} 
and the parallel pancakes
construction of~\cite{DKS17-sq} in Section~\ref{ssec:techniques}.) 
Specifically, we consider GMMs of the form $P =  \sum_{i=1}^k w_i \mathcal{N}(\bm \mu_i, \vec \Sigma_i)$, 
satisfying (a) $\min_i w_i \geq 0.9/k$, (b) $\vec \Sigma_i \preceq \vec I$, and 
(c)$\|\bm \mu_i - \bm \mu_j\|_2  \geq k^{\eps}$, for some $\eps>0$. 
Condition (a) posits that the component weights are nearly uniform. 
(This first condition is relevant for the clustering/parameter estimation problems, as these tasks 
require $\Omega(1/\min_i w_i)$ samples.)
Condition (b) says that each component covariance is unknown and bounded above
by the identity. Finally, condition (c) requires that the \new{component} 
means are pairwise separated by at least $k^{\eps}$,
in $\ell_2$-distance. Here the parameter $\eps>0$ is assumed to be sufficiently large 
so that $k^{\eps} \gg \sqrt{\log k}$. This assumption is required as, even for the uniform weights and 
identity covariance case (i.e., when $w_i = 1/k$ and $\vec \Sigma_i = \vec I$ for all $i$), 
the clustering problem can be solved with $\poly(d, k)$ samples if and only if 
the pairwise mean separation is  $\Delta \gg \sqrt{\log k}$~\cite{RV17-mixtures}.

It is easy to see that the aforementioned family of GMMs is learnable
using $\poly(d, k)$ samples (ignoring computational considerations).
Two independent works~\cite{HL18-sos, KSS18-sos} gave SoS-based learning algorithms for this family of GMMs
with sample complexity $k^{O(1)} d^{O(1/\eps)}$ and 
computational complexity $(d k)^{O(1/\eps^2)}$. With a more careful analysis, 
the runtime \new{can be further} improved to $(d k)^{O(1/\eps)}$~\cite{ST21, DKKPP22-dop}. 
Note that for the important special case
that the mean separation is $\Delta \gg \log^c(k)$, for some constant $c \geq 1/2$, 
these algorithms have quasi-polynomial sample and time complexities, namely $(dk)^{O(\log k)}$.

A natural question is whether \new{the aforementioned} 
upper bounds are inherent or can be significantly improved. 
Concretely, we address the following open problem:
\begin{center}
{\em Is there a $\poly(d, k)$-time learning algorithm for separated GMMs \\ 
with bounded covariance components and mean separation $\Delta = \polylog(k)$?}
\end{center}
For the special case of {\em spherical} components, 
namely when each individual Gaussian has identity covariance (i.e., $\vec \Sigma_i=\vec I$ for all $i$),
very recent work~\cite{LL22} made significant algorithmic progress on this question.
Specifically, they gave a  $\poly(d, k)$ time learning algorithm that succeeds
as long as $\Delta \gg \log^{1/2+c}(k)$, for any constant $c>0$. 
The algorithm in \cite{LL22} crucially leveraged
the assumption that the individual components are known (and equal to the identity). On the other hand,
their upper bound raised the hope that $\poly(d, k)$ complexity might be attainable even for unknown 
bounded covariance components with similar mean separation.

In this work, we provide evidence that known learning 
algorithms~\cite{HL18-sos, KSS18-sos, ST21,  DKKPP22-dop} for 
this subclass of GMMs are qualitatively best possible.
Concretely, we prove an SQ lower bound for this family of GMMs 
suggesting the following information-computation tradeoff: For mean separation $\Delta  = k^{\eps}$, 
any (SQ) learning algorithm either requires $2^{d^{\Omega(1)}}$ time 
or uses at least $d^{\Omega(1/\eps)}$ samples. 
In particular, this implies that the quasi-polynomial upper bounds 
for mean separation of $\Delta  = \polylog(k)$ are \new{best possible for the class of} SQ algorithms. 
Using known results~\cite{BBHLS20}, this SQ lower bound 
implies a qualitatively similar low-degree testing lower bound.

\new{We also provide an interesting implication for the special case of $\eps = 1/2$. 
Specifically, we establish an SQ lower bound suggesting 
that any efficient SQ algorithm under 
separation $\Delta \ll k^{1/2}$ requires nearly {\em quadratically} 
many samples (in the dimension $d$). 
On the other hand, $O(k d)$ samples suffice without computational limitations.
Recent work \cite{DKKLT21} developed an $O(dk)$-sample and computationally efficient 
algorithm for learning bounded covariance distributions (and, consequently, bounded covariance Gaussians) under separation $\tilde{\Omega}(k^{1/2})$. A natural open question is whether this 
separation bound can be significantly improved {\em while preserving sample near-optimality}. 
Perhaps surprisingly, we show that this is not possible for SQ algorithms: any efficient SQ algorithm
that works for separation $C k^{1/2}$, for a sufficiently small constant $C$, requires near-quadratically
many samples in $d$. This gap suggests that the algorithm of~\cite{DKKLT21} 
succeeds under the best possible separation within the class of computationally
efficient and sample near-optimal algorithms.}

\subsection{Our Results} \label{ssec:results}

Our main result is a Statistical Query lower bound of $d^{\Omega(1/\eps)}$ 
for learning the aforementioned subclass of Gaussian mixtures 
with mean separation $\Delta\geq k^\eps$. 

\new{Before we formally state our contributions, we require basic background
on the SQ model.}

\paragraph{SQ Model Basics} 
Before we state our main result, we recall the basics of the SQ model~\cite{Kearns:98, FGR+13}. 
Instead of drawing samples from the input distribution, 
SQ algorithms  are only permitted query access to the distribution via the following oracle:
\begin{definition}[VSTAT Oracle]\label{def:stat}
    Let $D$ be a distribution on $\R^d$. 
A statistical query is a bounded function $q: \R^d \to [0,1]$. 
For $u>0$, the $\mathrm{VSTAT}(u)$ oracle responds to the query $q$
with a value $v$ such that $|v - \E_{\bx \sim D}[q(\bx)]| \leq \tau$, where 
$\tau=\max\{ 1/u,\sqrt{\mathrm{Var}_{\bx \sim D}[q(\x)]/u} \}$. 
We call $\tau$ the \emph{tolerance} of the statistical query.
\end{definition}

An SQ lower bound for a learning problem $\Pi$ is typically of the following form: 
any SQ algorithm for $\Pi$ must either make a large number of queries $Q$ 
or at least one query with small tolerance $\tau$. 
When simulating a statistical query in the standard PAC model 
(by averaging i.i.d.\ samples to approximate expectations), 
the number of samples needed for a $\tau$-accurate query 
can be as high as $\Omega(1/\tau^2)$. Thus, we can intuitively interpret 
an SQ lower bound as a tradeoff between runtime of $\Omega(Q)$ 
or a sample complexity of $\Omega(1/\tau^{\new{2}})$.

\paragraph{Main Result}

Our main SQ lower bound result for learning GMMs 
is stated informally below. A more detailed formal version is provided in \Cref{thm:main}.

\begin{theorem}[Main Result, Informal] \label{thm:main-informal}
For $d,k \in \Z_+$ sufficiently large and $\eps>0$ such that $k^\eps \gg \sqrt{\log k}$, 
any SQ algorithm that correctly distinguishes between 
$\cN(\vec 0, \bI_d)$ and a $k$-GMM on $\R^d$ 
with minimum mixing weight at least $0.99/k$,
common covariance $\vec \Sigma \preceq \vec I_d$ for each component, 
and pairwise mean separation $\Delta \geq k^\eps$, 
either makes $2^{d^{\Omega(1)}}$ statistical queries or requires 
\new{at least one query to $\mathrm{VSTAT}(d^{\Omega(1/\eps)})$}.
\end{theorem}

As is typically the case, our SQ lower bound applies for the hypothesis testing problem
of distinguishing between the standard Gaussian and an unknown GMM in our family.
Hardness for testing a fortiori implies hardness for the corresponding learning problem 
(see \Cref{cor:density-estimation-hard}). 

A few additional remarks are in order. First notice that our SQ lower bound applies even for the special
case where the mixing weights are nearly uniform (within a factor of $2$, say)
and the component covariances are the same, 
as long as they are unknown\footnote{Recall that known algorithms do not require these assumptions. 
The runtime upper bound of $(dk)^{O(1/\eps)}$ holds as long as the minimum weight
is at least $1/\poly(k)$ and even if the component covariances are different.}. 
As it will become clear from our construction,  the common covariance matrix 
of each component  has only two distinct eigenvalues: each Gaussian component 
behaves like a standard Gaussian in all directions that are orthogonal 
to a low-dimensional subspace, and along that subspace behaves like 
 a spherical Gaussian with different variance. 
Finally, we remark that our lower bound applies for a large range of the parameter $\eps>0$, 
as long as $k^{\eps}$ is at least a sufficiently large constant multiple of $\sqrt{\log k}$. 
Consequently, it implies that the quasi-polynomial upper bounds for separation of $\polylog(k)$
are best possible for the class of SQ algorithms.

The implications of our SQ lower bound to the low-degree polynomial testing 
model, via the result of \cite{BBHLS20}, are provided in \Cref{sec:low_degree}.

\paragraph{Quadratic SQ Lower Bound for $\Omega(\sqrt{k})$ Separation}
\new{Our second result concerns the special case where the mean separation is proportional
to $k^{1/2}$, namely $C k^{1/2}$ for a sufficiently small universal constant $C$ (taking
$C=1/3$ suffices for our purposes). For this setting, 
we establish a nearly quadratic tradeoff between the sample complexity 
of the learning problem and the sample complexity of any efficient SQ algorithm for the problem.
Specifically, we show the following: }

\begin{theorem}[\new{Quadratic SQ Lower Bound}, Informal] \label{thm:sqrt-k-informal}
Let $d,k \in \Z_+$ with $d$ sufficiently large and $2 \leq k \ll \log d$. Any
SQ algorithm that correctly distinguishes between 
$\cN(\vec 0, \bI_d)$ and a $k$-GMM on $\R^d$ 
with uniform weights,
common covariance $\vec \Sigma \preceq \vec I_d$ for each component, 
and pairwise mean separation $\Delta \geq \sqrt{k}/\new{3}$, 
either makes $2^{d^{\Omega(1)}}$ statistical queries or requires 
\new{at least one query to $\mathrm{VSTAT}(d^{1.99})$}.
\end{theorem}

\new{A more detailed formal version is provided in \Cref{thm:epsilon_half}. 
The natural interpretation of the above result is as follows:
any SQ algorithm for this class of instances either uses $\Omega(d^{1.99})$ many samples
or requires at least $2^{d^{\Omega(1)}}$ many statistical queries (time). On the other hand, without
computational constraints, $O(k d)$ samples information-theoretically suffice.}

Using different techniques,~\cite{DDW21} established a low-degree testing lower bound
for the $k=2$ case with constant separation, 
suggesting a sample complexity tradeoff of $\tilde{\Omega}(d^2)$.

\subsection{Overview of Techniques} \label{ssec:techniques}

The best comparison to our results \new{is the prior work} of~\cite{DKS17-sq}. 
Both works prove SQ lower bounds for learning mixtures of separated, common covariance Gaussians. 
The major difference is that the~\cite{DKS17-sq} 
result requires large separation relative 
to the {\em smallest} eigenvalue of the covariance 
(or, more accurately, relative to the quadratic form defined by the inverse covariance matrix), 
while our result requires large separation relative to the {\em largest} eigenvalue. 
As we will see, this seemingly small distinction leads to significant differences.

Underlying both SQ lower bound results is the hidden-direction 
non-Gaussian component analysis construction of~\cite{DKS17-sq} 
(or, in our case, the generalization to hidden {\em subspaces} 
given in~\cite{diakonikolas2021optimality}). 
The high-level idea is that if one can find a distribution $A$ 
(defined in a small number of dimensions)
that matches its first $t$ moments with the standard Gaussian, 
then distinguishing the standard Gaussian from a distribution $D$ 
that behaves like $A$ along a hidden subspace 
and is standard Gaussian in the orthogonal directions 
requires SQ complexity $d^{\Omega(t)}$. \new{This generic result has been
leveraged to establish SQ lower bounds for a wide range of high-dimensional statistical tasks, 
see, e.g.,~\cite{DKS17-sq, DKS19, DKZ20,  GoelGK20, DK20-Massart-hard, diakonikolas2021optimality, DKKTZ21-benign, DKS18-list, DKPPS21, DiakonikolasKKZ20, Chen0L22}. 
The main difficulty in each case is, of course, 
to construct the desired moment-matching distributions.}

\new{In our context,} this means that for either result one needs to exhibit 
a distribution $A$, which is a mixture of $k$ separated Gaussians, 
so that $A$ matches many moments with the standard Gaussian. 
By letting $A$ be a discrete distribution with support size $k$ convolved with a narrow Gaussian, 
it suffices to find a distribution $A'$ supported on $k$ pairwise separated points so that $A'$ 
matches $t$ moments with a standard Gaussian.

At this point, the difference in the underlying separation assumptions becomes critical. 
In the parallel pancakes construction of~\cite{DKS17-sq}, 
one only needs the points 
in the support of $A'$ to have some minimal separation 
so that after convolving with a very narrow Gaussian, 
the resulting components of $A$ are still well separated in total variation distance. 
This \new{fact} allows them to use Gaussian quadrature and construct a {\em one-dimensional} distribution $A'$ 
which matches its first $t = 2k$ moments with $\cN(0, 1)$. This \new{construction} leads 
to an SQ lower bound of $d^{\Omega(k)}$. It should be noted that each 
unknown GMM in this old construction consists of $k$ ``skinny'' Gaussians 
whose mean vectors all lie in the same direction. Moreover, each pair of components 
will have total variation distance very close to $1$ and their 
mean vectors are separated by $\Omega(1/\sqrt{k})$.

In our setting however, we require much stronger separation assumptions. 
In particular, we require that the elements in the support of $A'$ be separated 
by some relatively large separation $\new{\Delta}$ on the order of $k^\eps \gg \sqrt{\log(k)}$. 
Unfortunately, it is provably impossible to find a moment-matching construction 
with this kind of separation in one dimension. Intuitively, this holds because 
the standard Gaussian $G \sim \cN(0,1)$ is highly concentrated about the origin. 
If $A'$ behaves similarly to $G$, it must also have most of its mass near the origin; 
but this is clearly impossible if the points of its support are pairwise separated by $\new{\Delta}$. 
More rigorously, one can show that 
the indicator function of an interval can be reasonably well-approximated 
by a constant-degree polynomial with respect 
to the Gaussian distribution (see, e.g,~\cite{DGJ+:10}). 
This implies that any distribution over $\R$ that matches constantly many moments with $G$ 
must be relatively close to $G$ in Kolmogorov distance, 
which is impossible for any discrete distribution with a widely separated support. 

To circumvent this issue, we instead produce a distribution $A'$ over $\R^m$, 
for some $m$ on the order of $\new{\Delta}^2$ (\Cref{prop:hard_instance}). 
Intuitively, this makes sense because Gaussian random points on $\R^m$ 
have pairwise separation approximately $\sqrt{m}=\new{\Delta}$;
this motivates us to use points drawn from $\cN(\vec 0, \bI_m)$ 
to construct the support of $A'$ (see \Cref{lem:hard-inst}, we will describe the construction in more detail in the next paragraph).
Unfortunately, this choice comes with a tradeoff. 
As the dimension of the space of degree-$t$ polynomials on $\R^m$ is approximately $m^t$, 
we will need the support of $A'$ to be of size roughly $m^t$ 
in order to have enough degrees of freedom to be able to match all of these moments. 
In particular, this means that the parameter $k$ needs to be on the order 
of $\new{\Delta}^{2t}$, 
and since we are considering separation $\new{\Delta} = k^\eps$, 
we need to choose $t$ to be on the order of $1/\eps$. 
Thus, the resulting SQ lower bound will be on the order of 
$d^{\Omega(t)}=d^{\Omega(1/\eps)}$.
Note that we cannot hope to do better, as the algorithms of~\cite{HL18-sos, KSS18-sos}
can be formalized as SQ algorithms with similar complexity.

It remains to explain how to construct $A'$. 
We want a distribution over a small support 
that matches $t$ moments with the standard Gaussian over $\R^m$ 
and also has large pairwise separation of its support points. 
The simple idea behind our construction 
is that picking a uniformly random set of points as our support 
should both ensure the separation 
with high probability, and also produce a set that is well-representative of a Gaussian. 
We achieve this as follows: we pick
an appropriate number of i.i.d.\ Gaussian random points in $\R^m$ and, 
using linear programming duality, show that with high probability there exists 
a moment-matching distribution supported on these points (cf. \Cref{lem:hard-inst}). 

\new{For the case of $\eps=1/2$ (which corresponds to pairwise mean separation of $\sim \sqrt{k}$),  
the above analysis is suboptimal because it shows an SQ lower bound of $d^{\Omega(1/\eps)}$ 
with the constant inside the big-$\Omega$ being rather large. In order to obtain a quadratic SQ lower bound for that case, we instead provide an explicit distribution over $\R^m$ matching three moments with the standard Gaussian (cf. \Cref{sec:sqrt-separation}).}

\section{Preliminaries} \label{sec:prelim}

We record the preliminaries necessary for the main body of this paper. We provide additional background in  \Cref{sec:additional_prelim}.

\subsection{Notation and Hermite Analysis}

\paragraph{Basic Notation}
 We use $\Z_+$ for positive integers and $[n] \eqdef \{1,\ldots,n\}$, $\cS^{d-1}$ for the $d$-dimensional unit sphere, and $\|\bv\|_2$ for the $\ell_2$-norm of a vectors. 
We use $\bI_d$ to denote the $d \times d$ identity matrix. 
For a matrix $\vec A$, we use $\|\vec A\|_\fr$ and $\|\vec A\|_{\op}$ to denote the Frobenius and spectral (or operator) norms respectively. 
We use $\cN(\bm{\mu}, \vec \Sigma)$ to denote the Gaussian with mean $\bm \mu$ and covariance matrix $\vec \Sigma$. For a set $S$, we use $\cU(S)$ for the uniform distribution on $S$. We use $\phi_m(\bx)$ for the pdf of the standard Gaussian in $m$-dimensions $\cN(\vec 0, \bI_m)$, and  $\phi(x)$ the pdf of $\cN(0,1)$. Slightly abusing notation, we will use the same letter for a distribution and its pdf, e.g., we will denote by $P(\bx)$ the pdf of a distribution $P$.\looseness=-1

	\paragraph{Hermite Analysis}  We use $h_k$ for the normalized probabilist's Hermite polynomials, which comprise a complete orthogonal basis of all functions $f:\R \to \R$ with $\E_{x\sim \cN(0,1)}[f^2(x)]< \infty$. When using multi-indices $\vec a \in \Z^d$ as subscripts, we refer to the multivariate Hermite polynomials.

\paragraph{Ornstein-Uhlenbeck Operator}
For a $\rho > 0$, we define the \emph{Gaussian noise} (or \emph{Ornstein-Uhlenbeck}) operator $U_\rho$ as the operator that maps a distribution $F$ on $\R^m$ to the distribution of the random variable $\rho \bx + \sqrt{1-\rho^2}\bz$, where $\bx \sim F$ and $\bz \sim \cN(\vec 0,\vec I_m)$ independently of $\bx$. 
A standard property of the $U_\rho$ operator is that it operates diagonally with respect to Hermite polynomials, i.e., $\E_{\bx \sim U_\rho F}[h_{\vec a}(\bx)] = \rho^{|\vec a|} \E_{\vec x \sim F}[h_{\vec a}(\bx)]$, where $|\vec a|=\sum_{i} a_i$.

\subsection{Background on the Statistical Query Model}\label{sec:sq-background}

We record the definitions from the SQ framework of~\cite{FGR+13} that we will need: We define the \emph{decision problem over distributions} $\mathcal{B}(\D, D)$ to be the hypothesis testing problem of distinguishing between $D$ and a member of the family of distributions $\D$. We define the \emph{pairwise correlation} between two distributions as $\chi_{D}(D_1, D_2) = \int_{\R^d} D_1(\bx) D_2(\x)/D(\bx)\, \d\bx - 1$. We say that a set of $s$ distributions $\mathcal{D} = \{D_1, \ldots , D_s \}$
 is $(\gamma, \beta)$-correlated relative to a distribution $D$
if $|\chi_D(D_i, D_j)| \leq \gamma$ for all $i \neq j$,
and $|\chi_D(D_i, D_j)| \leq \beta$ for $i=j$.

\begin{definition}[Statistical Query Dimension] \label{def:sq-dim}
Let $\beta, \gamma > 0$. Consider a decision problem $\mathcal{B}(\D, D)$,
where $D$ is a fixed distribution and $\D$ is a family of distributions. Define $s$ to be the maximum integer such that there exists a finite set of distributions
$\mathcal{D}_D \subseteq \D$ such that
$\mathcal{D}_D$ is $(\gamma, \beta)$-correlated relative to $D$
and $|\mathcal{D}_D| \geq s.$ The {\em Statistical Query dimension}
with pairwise correlations $(\gamma, \beta)$ of $\mathcal{B}$ is defined as $s$ and denoted as $\mathrm{SD}(\mathcal{B},\gamma,\beta)$.
\end{definition}

\begin{lemma}[Corollary 3.12 in \cite{FGR+13}] \label{lem:sq-from-pairwise}
Let $\mathcal{B}(\D, D)$ be a decision problem. For $\gamma, \beta >0$,
let $s= \mathrm{SD}(\mathcal{B}, \gamma, \beta)$.
For any $\gamma' > 0,$ any SQ algorithm for $\mathcal{B}$ requires queries \new{at least one query to $\mathrm{VSTAT}(1/(3(\gamma+\gamma')))$} or makes at least
$s  \gamma' /(\beta - \gamma)$ queries.
\end{lemma}

Our construction will use distributions that coincide with a given distribution $A$ in some subspace, and are standard Gaussians in every orthogonal direction. We need the following result from \cite{diakonikolas2021optimality} that upper bounds the correlation between two such distributions.

\begin{lemma}[Corollary 2.4 in \cite{diakonikolas2021optimality}] \label{lemma:correlation-bound}
Let  $A$ be a distribution over $\R^m$ such that the first $t$ moments of $A$
match the corresponding moments of $\normal(\vec 0,\vec I_m)$.
Let $G(\x){=}A(\x)/\phi_m(\x)$ be the ratio of the corresponding probability density functions.
For matrices $\vec U,  \vec V \in \R^{m\times d}$ such that $\vec U \vec U^\top =  \vec V \vec V^\top = \vec I_m$,
define $P_{A, \vec U}$ and $P_{A, \vec V}$ to be distributions over $\R^d$ with probability density functions
$G(\vec U\x)\phi_d(\bx)$ and $G(\vec V\x)\phi_d(\bx)$, respectively. Then, the following holds: 
$|\chi_{\normal(\vec 0,\vec I_m)}(P_{A, \vec U},P_{A, \vec V})| \leq \|\vec U\vec V^\top\|_\op^{t+1} \chi^2(A,\normal(\vec 0,\vec I_m))$.
\end{lemma}
Note that in the statement above,  $P_{A, \vec V}$ can be rewritten in the following form:
\begin{align}
    P_{A, \vec  V}(\bx) = A(\vec V \bx) \frac{\phi_d(\bx)  }{\phi_m(\vec V \bx)}
    = A(\vec V \bx) (2\pi)^{-\frac{(d-m)}{2}}  e^{-\frac{1}{2}\|\bx - \vec V^\top \vec V \bx\|_2^2}    = A(\vec V \bx) \phi_{d-m}\left( \mathrm{Proj}_{ \mathcal{V}^\perp}(\bx) \right) \;, \raisetag{2.5\baselineskip} \label{eq:hiden-dir}
\end{align}
where $ \mathrm{Proj}_{ \mathcal{V}^\perp}(\bx)=\bx - \vec V^\top \vec V \bx$ is the projection of $\bx$ to the subspace  that is perpendicular to the subspace $\mathcal{V}$ spanned by the rows of $\vec V$. Therefore,  \Cref{eq:hiden-dir} demonstrates that $P_{A, \vec V}$ coincides with the distribution $A$ in the subspace spanned by the rows of $\vec V$ and is standard Gaussian in the orthogonal complement.

\section{\new{Main Result:} Proof of \Cref{thm:main-informal}} \label{sec:main-proof}

In this section, we prove the following more detailed version of our main result (\Cref{thm:main-informal}).
Before moving to the proof, we state the implications of the above to the hardness 
of the corresponding density estimation problem in \Cref{cor:density-estimation-hard}. 

\begin{theorem}[SQ Lower Bound: Hypothesis Testing Hardness] \label{thm:main}
Let  $d,k \in \Z_+,\eps>0$ and $C$ be a sufficiently large absolute constant. 
\new{Assume that $k>(C/\eps)^{1/\eps}$, $d>k^{C\eps}$, and $k^{\eps} > C\sqrt{\log k}$}. 
\new{Consider} the following hypothesis testing problem regarding a distribution $P$ on $\R^d$: 
    \begin{itemize}[leftmargin=*]
        \item(Null Hypothesis) $P = \cN(\vec 0, \bI_d)$.
        \item(Alternative Hypothesis)  $P$ belongs to a family $\cP$, every member of which is a mixture 
        of Gaussians $\sum_{i=1}^k w_i \cN(\bm \mu_i, \vec \Sigma)$ for unknown \new{weights $w_i>0.99/k$},  
        mean vectors with pairwise separation $\|\bm \mu_i-\bm \mu_j\|_2 \geq k^{\eps}$ for all $i \neq j \in [k]$, 
        and common covariance matrix $\vec \Sigma \preceq \bI_d$. 
        \new{Moreover, $\dtv(P,\cN(\vec 0,\bI_d))>0.99$ and $\dtv(P,P')>0.99$ for all distinct $P,P' \in \cP$.}
    \end{itemize}
     Any algorithm with statistical query access to $P$ that distinguishes correctly between the two cases 
     does one of the following:  it performs $2^{d^{\Omega(1)}}$ statistical queries or uses at least 
     \new{one statistical query to $\mathrm{VSTAT}(d^{\Omega(1/\eps)} e^{-O(k^{2\eps})})$.}
    \end{theorem}

    \begin{corollary}[SQ Lower Bound: Density Estimation Hardness]\label{cor:density-estimation-hard}
        Under the assumptions of \Cref{thm:main} and the additional assumption $k^{\eps}<\sqrt{\log(d)/(C\eps)}$, 
        let $\cA$ be an SQ algorithm that given access to a mixture of Gaussians 
        $P=\sum_{i=1}^k w_i \cN(\bm \mu_i, \vec \Sigma)$ for some unknown weights $w_i>0.99/k$,  
        mean vectors $\bm \mu_i \in \R^d$ for $i\in [k]$ with pairwise separation 
        $\|\bm \mu_i-\bm \mu_j\|_2 \geq k^{\eps}$ and common covariance matrix $\vec \Sigma \preceq \bI_d$, 
        finds a distribution $Q$ with $\dtv(P,Q)<1/4$. Then $\cA$ necessarily does one of the following: 
        it performs $2^{d^{\Omega(1)}}$ statistical queries or uses at least \new{one statistical query to 
        $\mathrm{VSTAT}(d^{\Omega(1/\eps)} e^{-O(k^{2\eps})})$}.
    \end{corollary}
    \begin{proof} 
        The reduction from the hypothesis testing problem of \Cref{thm:main} to the corresponding learning problem is fairly standard, 
        see, e.g., Lemma 8.5 in \cite{diakonikolas2020robust}. To check the applicability of that lemma, 
        we note that $\dtv(P,\cN(\vec 0, \bI_d)) > 0.99 > 2(\tau + 1/4)$, 
        where the inequality uses the assumption $k^{\eps}<\sqrt{\log(d)/(C\eps)}$ 
        for bounding the query tolerance $\tau$ by a constant. \looseness=-1
    \end{proof}

The main ingredient towards proving \Cref{thm:main} is \Cref{prop:hard_instance}, 
which establishes the existence of a low-dimensional spherical $k$-GMM with well-separated means, 
that matches its first $\Omega(1/\eps)$ moments with the standard Gaussian. 
We prove this result in \Cref{sec:hard-instance}. 
In this section, we show how \Cref{thm:main} follows from \Cref{prop:hard_instance}.  
    
\begin{proposition} \label{prop:hard_instance}
Let $\eps > 0$, $d,k \in \Z_+$, $c>0$ be a sufficiently small constant and $C$ be a sufficiently large constant. 
If \new{$k>(C/\eps)^{1/\eps}$, $d>k^{C\eps}$, and $k^{\eps} > C\sqrt{\log k}$}, 
there exists a distribution $A$ over $\R^{m}$ with $m:=k^{2\eps}$ that satisfies the following:
\begin{enumerate}[leftmargin=*, label = (\roman*)]
    \item $A$ is a mixture of $k$ spherical Gaussians in $\R^m$ with variance \new{$\delta = c k^{-2.5/m}$} 
    \blue{in every direction} \new{and minimum mixing weight at least $0.99/k$}. \label{it:GMMness}
    \item $A$ matches its first \new{$t=\Theta(1/\eps)$} moments with $\cN(\mathbf{0},\vec I_m)$.\label{it:moment-matching}
    \item The means $\bm \mu_i,\bm \mu_j$ of any two distinct components have separation $\| \bm \mu_i - \bm \mu_j \|_2  \geq k^\eps$. \label{it:separation}
    \item For every $\bU,\bV \in \R^{m \times d}$ with $\bU\bU^\top = \bV\bV^\top= \bI_d$  and $\|\vec U \vec V^\top \|_\fr = O(d^{-\frac{1}{10}})$, it holds $\dtv(P_{A,\bU}, P_{A,\bV}) > 0.99$. 
    Moreover, for all  $\bV\in \R^{m \times d}$ it holds $\dtv(P_{A,\bV}, \cN(0,\bI_d)) > 0.99$.\looseness=-1\label{it:dtv-separation-or}
    \item $\chi^2(A,\cN(\vec 0,\vec I_m) \leq \delta^{-m/2}e^{O(m)}$. \label{it:chi-square}
\end{enumerate}
\end{proposition}

To prove \Cref{thm:main}, we create a family of distributions of the form of \Cref{eq:hiden-dir} by embedding the $k$-GMM onto many nearly orthogonal subspaces. The resulting distributions in $\R^d$ will be the $k$-GMMs  described in our main theorem's statement. We then use the properties established in \Cref{prop:hard_instance} to argue that this family has a large SQ dimension, making it hard to learn.\looseness=-1

\begin{proof}[Proof of \Cref{thm:main}]
Recall the definition of a \emph{decision problem} over distributions (\Cref{def:decision}). 
Consider the decision problem $\mathcal{B}(\D, D)$, 
where $D = \cN(\vec 0, \bI_d)$ and $\D$ is defined to be the set of distributions 
of the form $P_{A, \bV}$ as in \Cref{eq:hiden-dir}.     
We bound from below the SQ dimension (\Cref{def:sq-dim}) of $\mathcal{B}(\D, D)$. 
Let $S$ be the set from the fact below.
\begin{fact}[See, e.g., Lemma 17 in \cite{diakonikolas2021optimality} ] \label{fact:setofmatrices}
Let $m,d \in \N$ with $m<d^{1/10}$. There exists a set $S$ of $2^{d^{\Omega(1)}}$ matrices 
in $\R^{m \times d}$ such that every $\bU \in S$ satisfies $\bU \bU^\top = \bI_m$ 
and every pair $\bU,\bV \in S$ with $\bU \neq \bV$ satisfies $\|\bU \bV^\top \|_\fr \leq O(d^{-1/10})$.
\end{fact}
Let $\mathcal{D}_D := \{ P_{A, \bV} \}_{\bV \in S}$.  
Using \Cref{fact:setofmatrices} and \Cref{lemma:correlation-bound}, 
we have that for any distinct $\vec V, \vec U \in S$ 
\begin{align}
        |\chi_{\normal(\vec 0,\vec I_m)}(P_{A, \vec U},P_{A, \vec V})| 
        &\leq \left\|\vec U\vec V^\top  \right\|_\op^{t+1} \chi^2(A,\normal(\vec 0,\vec I_m))
        \leq \Omega(d)^{-(t+1)/10} \chi^2(A,\normal(\vec 0,\vec I_m)) \;, \label{eq:sub-optimal}
    \end{align}
    where we used that $\| \vec A \|_\op \leq \|\vec A \|_\fr$ for any matrix $\vec A$.
    On the other hand, when $\vec V = \vec U$, we have that 
    $ |\chi_{\normal(\vec 0,\vec I_m)}(P_{A, \vec U},P_{A, \vec V})| \leq \chi^2(A,\normal(\vec 0,\vec I_m)) $. 
    Thus, the family $\mathcal{D}_D$ is 
    $(\gamma,\beta)$-correlated with $\gamma = \Omega(d)^{-(t+1)/10} \chi^2(A,\normal(\vec 0,\vec I_m)) $ 
    and $\beta = \chi^2(A,\normal(\vec 0,\vec I_m)) $ with respect to $D=\cN(\vec 0, \bI_m)$.
    This means that $\mathrm{SD}(\mathcal{B}(\D, D),\gamma,\beta)\geq \exp({d^{\Omega(1)}})$. 
 
    Recall that $t=\Theta(1/\eps)$. Applying \Cref{lem:sq-from-pairwise} with 
    $\gamma' := \gamma = \Omega(d)^{-(t+1)/10} \chi^2(A,\normal(\vec 0,\vec I_m))$, 
    we obtain that any SQ algorithm for $\mathcal{Z}$ requires at least 
    $\exp(d^{\Omega(1)}) d^{-O(t)} = \exp(d^{\Omega(1)}) d^{-O(1/\eps)} $ calls to 
    \begin{align*}
        \mathrm{VSTAT}\left( d^{\Omega(1/\eps)}/ \chi^2(A,\normal(\vec 0,\vec I_m)) \right) \;.
    \end{align*}   
    Finally, using \Cref{prop:hard_instance}, 
    $\chi^2(A,\normal(\vec 0,\vec I_m)) \leq k^{O(1)}\exp(O(m)) =  k^{O(1)}\exp(O(k^{2\eps})) \leq \exp(O(k^{2\eps}))$, where we also used our assumption that $k^{\eps}$ is much bigger than $\sqrt{\log k}$. 
The number of calls $\exp(d^{\Omega(1)}) d^{-O(1/\eps)}$ mentioned before can be bounded below 
by $\exp(d^{\Omega(1)})$, 
using our assumptions that $d>k^{C\eps}> (C/\eps)^C$. 
This completes the proof of \Cref{thm:main}. 
\end{proof}

\subsection{Moment Matching: Proof of \Cref{prop:hard_instance}}\label{sec:hard-instance}

In \Cref{sec:duality}, we provide the basis for \Cref{prop:hard_instance}, 
which shows the existence of a low-dimensional \emph{discrete} distribution using an LP-duality argument. 
Then, in \Cref{sec:wrap-up}, we complete the proof of \Cref{prop:hard_instance}.

\subsubsection{LP Duality Argument} \label{sec:duality}

We establish the following:

\begin{proposition}\label{lem:hard-inst}
Let $C$ be a sufficiently large absolute constant. 
\new{For any  $m,t \in \Z_+$  with $m>Ct^2$,}  
there exists a discrete distribution $D$ on $\R^m$ with support $\supp(D)$ satisfying the following: 
\begin{enumerate}[label = (\roman*)]
    \item \new{$|\supp(D)| = m^{13t}$},  \label{it:sample-complx}
    \item \new{$D$ gives mass at least $0.99/|\supp(D)|$ to every point in its support,} \label{it:uniformity}
        \item $D$ matches its first $t$ moments with $\cN(\vec 0,\vec I_m)$, 
        i.e., $\E_{\bx \sim D}[p(\bx)] = \E_{\bx \sim \cN(\vec 0, \bI)}[p(\bx)]$, 
        for every polynomial $p: \R^{\new{m}} \to \R$ of degree at most $t$,  \label{it:duality}
        \item $0.9 \sqrt{m} \leq  \|\x\|_2 \leq 1.1 \sqrt{m}$ for all $\x  \in \supp(D)$. \label{it:norm}
    \item for any distinct $\x,\vec y \in \supp(D)$ it holds $\|\vec x - \vec y \|_2 \geq  \sqrt{m}$. \label{it:pairwise-dist}
\end{enumerate}
\end{proposition}
\begin{proof}
Let  a set $S=\{\x_1,\ldots,\x_N\}$ of $N = m^{13t}$ points drawn from $\cN(\vec 0,\vec I_m)$. 
We will show that with non-trivial probability, taking $D$ to be the uniform distribution over $S$ 
satisfies the desired properties.
The proof is based on an LP duality argument.
Proving \Cref{it:uniformity,it:duality} is equivalent to proving that the linear program below (with unknowns $\{ \mu_i\}_{i \in [N]}$) admits a solution. \new{  Let $\alpha:=0.99/N$, the desired lower bound for all weights. The LP is the following:}
\begin{alignat}{2}
    & \text{Find:}& \quad &\mu_1,\ldots,\mu_N\notag\\
   & \text{s.t.: }& \quad & \begin{aligned}[t]
                   \sum_{i \in [n]} \mu_i p(\bx_i)  &= \E_{\bx \sim \cN(\vec 0,\vec I_m)}[p(\bx)] ,& \text{ for any at most $t$-degree polynomial}\; p \\[1ex]   
\mu_i  &\geq \alpha, \quad \text{ for all }i\in [N]
                \end{aligned}\label[lp]{lp:primal}
\end{alignat}
Note that the first constraint for $p$ being the constant polynomial $p = 1$ means that the $\mu_i$'s form a valid distribution.
By standard LP duality, the above is feasible unless there exists a linear combination of constraints that produces the contradiction $0<-1$. 
Concretely, we start by introducing multipliers, also known as dual variables, for each constraint. For the final constraint, these will be some variables  $\beta_i \geq 0$ for $i \in [N]$. Regarding the first constraint,  a multiplier from $\R$ is assigned to every polynomial with a degree of at most $t$. However, since the first constraint applies to all such polynomials and the set is closed under multiplication, these dual variables can be absorbed into the polynomials and will not be explicitly written. After multiplying and summing the constraints, we obtain \looseness=-1 
\begin{align} \label{eq:adding_them_up}
    \sum_{i \in [N]}\mu_i\left( - \beta_i + p(\x_i)  \right) \leq  \E_{\bx \sim \cN(\vec 0,\vec I_m)}\left[p(\x) \right] - \alpha \sum_{i \in [N]} \beta_i  \;.
\end{align}
To derive the dual LP, we set the coefficients of $\mu_i$ equal to zero and ask for the right-hand side of \Cref{eq:adding_them_up} to be negative. This means that the primal  \Cref{lp:primal} is feasible unless \Cref{lp:dual} on the left part below has a solution, where \Cref{lp:dual} is further equivalent to \Cref{lp:dual2} on the right part:\looseness=-1
\vspace{-10pt}
\hspace{-\leftmargin}
\begin{minipage}[t]{0.45\textwidth}
\begin{alignat}{2}
    & \text{Find:}& \quad &\beta_1,\ldots,\beta_N\in \R_+, \notag\\
    &             &       &p\;  \text{at most $t$-degree polynomial}\notag\\
   & \text{s.t.: }& \quad & \hspace{-20pt}\begin{aligned}[t]
                  - \beta_i  + p(\x_i) &= 0, \; \forall i\in[N] \\[1ex]
                 \E_{\bx \sim \cN(\vec 0,\vec I_m)}\left[ p(\bx)\right] &<  \alpha \sum_{i \in [N]} \beta_i
                \end{aligned}\label[lp]{lp:dual}
\end{alignat}
\end{minipage} \hspace{40pt}
\begin{minipage}[t]{0.45\textwidth}
\begin{alignat}{2}
    & \text{Find:}& \quad &p\;  \text{at most $t$-degree polynomial}\notag\\
    &             &       & \notag\\
   & \text{s.t.: }& \quad & \hspace{-30pt}\begin{aligned}[t]
                p(\x_i) &\geq 0,  \; \forall i\in[N]  \\[1ex]
                 \E_{\bx \sim \cN(\vec 0,\vec I_m)}\left[ p(\bx)\right] &< \alpha  \hspace{-1pt} \cdot N  \hspace{-3pt} \cdot \hspace{-10pt} \E_{\bx \sim \cU(S)}[p(\bx)]
                \end{aligned}\label[lp]{lp:dual2}
\end{alignat}
\end{minipage}
\vspace{7pt}

For verifying the equivalence of the two LPs it suffices to note that $\sum_{i \in [N]} \beta_i = \sum_{i \in [N]} p(\bx_i) = N \E_{\bx \sim \cU(S)}[p(\bx)]$. By scaling (homogeneity), we can assume in the above that $\E_{\bx \sim \cN(\vec 0,\vec I_m)}[p^2(\bx)]=1$.
Recall that the points $\bx_1,\ldots,\bx_N$ are samples from $\cN(\vec 0, \vec I_m)$. Since we are proving the proposition via a probabilistic argument, it remains to show that with non-trivial probability 
these points will be such that \Cref{lp:dual2} is infeasible (and thus \Cref{lp:primal} is feasible). 
We prove this by contradiction: Assume that \Cref{lp:dual2} is feasible. Let $\cU(S)$ be the uniform distribution over $S$. We show that, in fact, $\cU(S)$ approximates the first four moments of any polynomial with a degree at most $t$ (the proof is given in \Cref{sec:concentration}). Formally, we show the following:

\begin{restatable}{claim}{MOMENTS}\label{clm:empirical-moments} 
Let a set $S=\{\bx_1,\ldots,\bx_N \}$ of i.i.d.\ samples $\bx_i \sim \cN(\vec 0, \bI_m)$. 
If \new{$N>10 m^{12t}/\eta^2$}, then with probability at least $0.6$, 
for any polynomial $p:\R^m \to\R$ of degree at most $t$ it holds \looseness=-1
    \begin{enumerate}[label = (\roman*)]
    \item $\E_{\x\sim \cU(S)}[p(\bx)] \leq \E_{\bx \sim \cN(\vec 0,\vec I_m)}[p(\bx)] + \eta$,
    \item $\E_{\x\sim \cU(S)}[p^2(\bx)]  \geq \E_{\bx \sim \cN(\vec 0,\vec I_m)}[p^2(\bx)] -  \eta$, and
    \item $\E_{\x\sim \cU(S)}[p^4(\bx)]  \leq \E_{\bx \sim \cN(\vec 0,\vec I_m)}[p^4(\bx)] +  \eta$.
\end{enumerate}
\end{restatable}

For our case, we assumed that \Cref{lp:dual2} is feasible, 
thus $\E_{\bx \sim \cN(\vec 0,\vec I_m)}[p(\bx)] <a N  \E_{\bx \sim \cU(S)}[p(\bx)] $. 
We use \Cref{clm:empirical-moments} with accuracy $\eta= 3^{-t }/200$, 
so the sample complexity from that claim becomes $40000 \cdot 9^t m^{12t}$. 
Since we assumed that $m>6000$, the number of samples that we use 
is $N=m^{13t} > 40000 \cdot 9^t m^{12t}$ 
and thus satisfies the requirement of the claim. 
The claim thus yields 
\begin{align*}
    \E_{\bx \sim \cU(S)}[p(\bx)] \leq  \E_{\bx \sim \cN(\vec 0,\vec I_m)}[p(\bx)] + \eta
    <a N  \E_{\bx \sim \cU(S)}[p(\bx)] + \eta \;,
\end{align*}
which means that
\begin{align}\label{eq:ononehand}
    \E_{\bx \sim \cU(S)}[p(\bx)] <\frac{\eta}{1-a N} \leq \frac{3^{-t }/200}{1-0.99}=  \frac{3^{-t}}{2} \;.
\end{align}
On the other hand, for every $t\geq 1$ we have that
\begin{align}
   \E_{\bx \sim \cU(S)}[p(\bx)] 
   \geq \frac{\E_{\bx \sim \cU(S)}[p^2(\bx)]^{3/2} }{\sqrt{\E_{\bx \sim \cU(S)}[p^4(\bx)]} } 
   \geq \frac{(1-\eta)^{3/2}}{\sqrt{\E_{\bx \sim \cN(\vec 0,\vec I_m)}[p^4(\bx)] + \eta}}
    \geq     \frac{0.7}{\sqrt{3^{2t} + 3^{-t }/2}} \geq \frac{3^{-t}}{2}   \;, \label{eq:oneotherhand}
\end{align}
where the penultimate inequality uses Gaussian hypercontractivity (\Cref{lem:hypercontractivity}). 
Comparing \Cref{eq:ononehand,eq:oneotherhand}  we have obtained a contradiction.

We now show the lower bound of \Cref{it:norm}. 
Using the concentration of the norm of a Gaussian vector (\Cref{lem:norm-conc} with $\beta=\sqrt{m}/10$), 
we have that
\begin{align}\label{eq:badevent1}
\pr_{\bx_1, \ldots, \bx_N \sim \cN(\vec 0, \vec I_m)}[ \exists i : \lvert   \|\bx_i\|_2  - \sqrt{m} \rvert  < 0.1 \sqrt{m}] 
\leq 2Ne^{-m/1600} = 2 m^{13t} e^{-m/1600} <0.1 \;,
\end{align}
where we used that $t< \sqrt{m}/16000 < \frac{m/1600-\ln(20)}{13\ln m}$ for $m>30000$.

Regarding \Cref{it:pairwise-dist}, it is a standard property of the Gaussian all pairs 
of points are nearly-orthogonal with high probability (\Cref{fact:orthogonality} with $\alpha=0.1$),  
\begin{align} \label{eq:badevent2}
    \pr_{\bx_1, \ldots, \bx_N \sim \cN(\vec 0, \vec I_m)}[ \exists i\neq j :  |\langle \bx_i,\bx_j \rangle| > m^{-0.1}] \leq N^2 e^{-m^{0.8}/5} 
    \leq m^{26t} e^{-m^{0.8}/5} < 0.1 \;,
\end{align}
where the last inequality uses that $t< \sqrt{m}/16000 < \frac{m^{0.8}/5-\ln (10) }{26 \ln m}$ for $m>30000$.
Conditioning on the two bad events of \Cref{eq:badevent1,eq:badevent2} not happening, we have that for any distinct $i,j \in [N]$, it holds $\|\bx_i - \bx_j \|_2^2 = \|\bx_i\|_2^2 + \|\bx_j\|_2^2 - 2 \langle \bx_i,\bx_j \rangle \geq 1.62m - 2m^{-0.1} \geq m$, for $m>2$.
\end{proof}

\subsubsection{Proof of \Cref{prop:hard_instance}} \label{sec:wrap-up}
We use the following throughout the proof: 
Let $D$ be the distribution from \Cref{lem:hard-inst} with parameters $m=k^{2\eps}$, 
and $t=1/(26\eps)$ \new{(note that because of our assumption  $k>(C/\eps)^{1/\eps}$ 
the requirement of \Cref{lem:hard-inst} is satisfied and thus the proposition is applicable). }
Let $A = U_\rho(D)$, where $U_\rho$ denotes the Ornstein-Uhlenbeck operator. 
We choose \new{$\rho = \sqrt{1-\delta}$} and $\delta = c k^{-2.5/m}$ 
for a sufficiently small positive constant $c$.
We prove each part of \Cref{prop:hard_instance} separately.

\medskip

\noindent\textbf{Proof of \Cref{it:GMMness}}
The fact that $A$ is a mixture of Gaussians with each component having variance $\delta$ \new{in each direction}
follows immediately by the definition of $A$ as the distribution $D$ after Gaussian smoothing 
via the Ornstein-Uhlenbeck operator with parameter $\rho = \sqrt{1-\delta}$. 
We can also check that the number of components is $k$: 
by \Cref{lem:hard-inst} we have that the number of components is $m^{13 t}$. 
Recall that we have further selected $m=k^{2\eps}$. Thus, the number of components 
is $m^{13 t}=k^{26\eps t}$. This is equal to $k$ by our choice of $t = 1/(26\eps)$. 
The fact that we have mass $0.99/k$ for each Gaussian component 
follows from \Cref{it:uniformity} of \Cref{lem:hard-inst}.

\medskip

\noindent \textbf{Proof of \Cref{it:moment-matching}}
For any $\vec a \in \N^m$ with $|\vec a|\leq t$, we have
\begin{align*}
    \E_{\bx \sim U_\rho (D)}[h_{\vec a}(\bx)] = \rho^{|\vec a|} \E_{\bx \sim D}[h_{\vec a}(\bx)] = \rho^{|\vec a|}\E_{\bx \sim \cN(\vec 0, \vec I_m)}[h_{\vec a}(\bx)]= \E_{\bx \sim \cN(\vec 0, \vec I_m)}[h_{\vec a}(\bx)] \;,
\end{align*}
where the first equality uses \Cref{fact:eigenfunction}, 
the next one uses \Cref{it:duality} of \Cref{lem:hard-inst}, 
and the last one is due to the property of Hermite polynomials 
$\E_{\x \sim \cN(\vec 0, \vec I)}[h_{\vec a}(\x)] = 1$ if $|\vec a|=0$ and zero otherwise.

\medskip

\noindent\textbf{Proof of \Cref{it:separation}}
Using \Cref{it:pairwise-dist}  of \Cref{lem:hard-inst} combined with our choice $m=k^{2\eps}$ 
and the fact that the Ornstein-Uhlenbeck operator scales all the means 
by a factor of $\rho=\sqrt{1-\delta}>1/2$, we will have that the pairwise means separation 
in our construction is at least $\rho k^{\eps} >  k^{\eps}/2$.

\medskip

\noindent\textbf{Proof of \Cref{it:dtv-separation-or}}
We start with some notation. 
Denote by $\mathcal{V}$ the subspace spanned by $\{\bv_1,\ldots,\bv_m \}$, 
and $\mathcal{U} = \mathrm{span}\{\bu_1,\ldots,\bu_m \}$. 
Extend the set $\bv_1,\ldots,\bv_m$ to an orthonormal basis  $\bv_1,\ldots,\bv_m, \bv_{m+1},\ldots,\bv_{2m}$ 
of the vector space spanned by the vectors $\{ \bv_1,\ldots,\bv_m, \bu_1, \ldots, \bu_m\}$. 
Furthermore, let the vectors $\bv_1,\ldots,\bv_{2m}\ldots,\bv_{d}$ be the extension 
to an orthonormal basis of the entire $\R^d$. Consider the matrices 
$\vec R_{\vec V_1}=[\bv_1 \ldots \bv_m]^\top$ 
(note that $\vec R_{\vec V_1}$ coincides with $\bV$ in this notation), 
$\vec R_{\vec V_2}=[\bv_{m+1} \ldots \bv_{2m}]^\top$, and 
$\vec R_{\vec V_3}=[\bv_{2m+1} \ldots \bv_d]^\top$. 
Let $\vec R_{\vec V}=[\vec R_{\vec V_1}^\top \vec R_{\vec V_2}^\top \vec R_{\vec V_3}^\top ]^\top$. 

We also define a similar notation regarding $\vec U$. 
Namely, extend the set $\bu_1,\ldots,\bu_m$ to an orthonormal basis 
$\bu_1,\ldots,\bu_m, \bu_{m+1},\ldots,\bu_{2m}$ of the vector space 
spanned by the vectors $\{ \bv_1,\ldots,\bv_m$, $\bu_1, \ldots, \bu_m\}$. 
Let $\bu_1,\ldots,\bu_{2m}\ldots,\bu_{d}$ be its extension to an orthonormal basis 
of the entire $\R^d$. Define the matrices 
$\vec R_{\vec U_1}=[\bu_1 \ldots \bu_m]^\top$, 
$\vec R_{\vec U_2}=[\bu_{m+1} \ldots \bu_{2m}]^\top$, 
and $\vec R_{\vec U_3}=[\bu_{2m+1} \ldots \bu_d]^\top$. 
Let $\vec R_{\vec U}=[\vec R_{\vec U_1}^\top \vec R_{\vec U_2}^\top \vec R_{\vec U_3}^\top ]^\top$. 
Since  $\vec R_{\vec U_3}$ and $\vec R_{\vec V_3}$ are meant to be orthonormal bases of the same space, 
we pick  $\vec R_{\vec U_3} = \vec R_{\vec V_3}$.

We now focus on our integral:
\begin{align} \label{eq:integral}
    \mathcal{I}_{\vec V,\vec U} &\eqdef\int_{\bz\in \R^{d}} \min \{ P_{A,\vec V}(\bz),  P_{A,\vec U}(\bz)\} \d \bz \;,
\end{align}
where $P_{A,\bV}$ and $P_{A,\bU}$ are defined as in \Cref{eq:hiden-dir} (where recall that $\phi_k$ denotes the pdf of the $k$-dimensional standard Gaussian). Using that definition for $P_{A,\bV}$ and the notation that we introduced earlier, we write 
\begin{align*}
    P_{A,\bV}(\bz) &= A(\vec V\bz) \phi_{d-m}\left( \mathrm{Proj}_{\mathcal{V}^\perp}(\bz) \right) \\
    &= A(\vec V\bz) \phi_{d-m}\left([ \vec R_{\vec V_2}^\top \vec R_{\vec V_3}^\top ]^\top \bz \right) \\
    &= A(\vec V\bz) \phi_{m}\left( \vec R_{\vec V_2}  \bz \right) \phi_{d-2m}\left(  \vec R_{\vec V_3}  \bz  \right),
\end{align*}
where in the last equality we separated the standard Gaussian into two components. Using a similar rewriting for $ P_{A,\bU}(\bz)$ along with $\vec R_{\vec U_3} = \vec R_{\vec V_3}$ (see first paragraphs), our integral becomes
\begin{align*}
    \mathcal{I}_{\vec V,\vec U} = \int_{\bz\in \R^{d}} \min \{ 
    A(\vec V\bz) \phi_{m}\left( \vec R_{\vec V_2}  \bz \right) \phi_{d-2m}\left(  \vec R_{\vec V_3}  \bz  \right),
    A(\vec U\bz) \phi_{m}\left( \vec R_{\vec U_2}  \bz \right) \phi_{d-2m}\left(  \vec R_{\vec V_3}  \bz  \right)
    \} \d \bz \;.
\end{align*}

We rotate the space using the unitary matrix $\vec R_{\vec V}^\top$.  Hence,  \Cref{eq:integral} becomes
\begin{align} 
    \mathcal{I}_{\vec V,\vec U} = \int_{\bz\in \R^{d}} \min \{ &A(\vec V\vec R_{\vec V}^\top \bz) \phi_{m}(\vec R_{\vec V_2} \vec R_{\vec V}^\top \bz) \phi_{d-2m}( \vec R_{\vec V_2} \vec R_{\vec V}^\top \bz) , \notag \\
    &A(\vec U \vec R_{\vec V}^\top \bz) \phi_{m}\left( \vec R_{\vec U_2} \vec R_{\vec V}^\top \bz \right) \phi_{d-2m}\left(  \vec R_{\vec V_3}  \vec R_{\vec V}^\top \bz  \right) \} \d \bz \;. \label{eq:after-rot}
\end{align}
By definition of these matrices, it holds that $\vec V\vec R_{\vec V}^\top = [\bI_{m \times m} \; \vec 0_{m \times (d-m)}]$. Similarly it holds $\vec R_{\vec V_2} \vec R_{\vec V}^\top = [\vec 0_{m \times m}\; \bI_{m \times m} \; \vec 0_{m \times (d-2m)}]$, and $\vec R_{\vec V_3} \vec R_{\vec V}^\top = [\vec 0_{(d-2m) \times 2m} \; \bI_{(d-2m) \; \times (d-2m)}]$. Using the notation $\bx_{1\ldots k} = (x_1,\ldots, x_k)$ to denote the first $k$ coordinates of a vector $\bx \in \R^{d}$ with $d\geq k$, we have that  $\vec V\vec R_{\vec V}^\top \bz = \bz_{1\ldots m}$, and similarly $\vec R_{\vec V_2} \vec R_{\vec V}^\top \bz = \bz_{m+1\ldots 2m}$, $\vec R_{\vec V_3}  \vec R_{\vec V}^\top \bz = \bz_{2m+1\ldots d}$. Using that simplification and renaming $\bx = \bz_{1\ldots m}$, $\by = \bz_{m+1 \ldots 2m}$, $\bw = \bz_{2m+1\ldots d}$, the first part of the min operator in \Cref{eq:after-rot} can be rewritten as  $A(\vec V\vec R_{\vec V}^\top \bz) \phi_{m}(\vec R_{\vec V_2} \vec R_{\vec V}^\top \bz) \phi_{d-2m}( \vec R_{\vec V_2} \vec R_{\vec V}^\top \bz) = A(\bx) \phi_m(\by) \phi_{d-2m}(\bw)$. Using similar reasoning for the second part of the min, we have that
\begin{align} 
    \mathcal{I}_{\vec V,\vec U} &= \int \min \{ A(\bx) \phi_m(\by) \phi_{d-2m}(\bw), \notag \\
   &\quad \quad   A(\vec U \vec R_{\vec V_1}^\top \bx + \vec U \vec R_{\vec V_2}^\top \by) \phi_m(\vec R_{\vec U_2} \vec R_{\vec V_1}^\top \bx + \vec R_{\vec U_2} \vec R_{\vec V_2}^\top \by) \phi_{d-2m}(\bw)
    \} \d \bx \d \by \d \bw  \notag \\
    &= \int_{\bz\in \R^{d}} \min \{ A(\bx) \phi_m(\by),
    A(\vec U \vec R_{\vec V_1}^\top \bx + \vec U \vec R_{\vec V_2}^\top \by) \phi_m(\vec R_{\vec U_2} \vec R_{\vec V_1}^\top \bx + \vec R_{\vec U_2} \vec R_{\vec V_2}^\top \by)
    \} \d \bx \d \by\;, \label{eq:integral2}
\end{align} 
where the last line takes $\phi_{d-2m}(\bw)$ as a common factor and uses that its integral with respect to $\bw$ equals to one.
We now do the following change of integration variables:
\begin{align*}
    \begin{bmatrix}
\bx \\
\bx'\\
\end{bmatrix}
=
    \begin{bmatrix}
\bI & \mathbf{0} \\
\vec U\vec R_{\vec V_1}^\top  &\vec U\vec R_{\vec V_2}^\top  \\
\end{bmatrix}
    \begin{bmatrix}
\bx \\
\by\\
\end{bmatrix}\;.
\end{align*}
The Jacobian of the inverse transformation is $1/\det(\vec U\vec {R_{\vec V}}_2^\top)$, 
where we used the fact that $\det(\vec A^{-1}) = 1/\det(\vec A)$ as well as the fact that 
(due to the identity block of the matrix) the determinant ends up being only that of the bottom right block.

Performing this change of variables in \Cref{eq:integral2} 
and using the pointwise upper bound $\phi_m(\cdot) \leq (2\pi)^{-m/2}\leq 1$, 
we obtain
\begin{align} \label{eq:changed_int}
   \mathcal{I}_{\vec V,\vec U}
    &\leq  \frac{1}{\det(\vec U{\vec R_{\vec V}}_2^\top)} \int_{\bx\in \R^{m}} \int_{\bx'\in \R^{m}}\min \{ A(\bx) , A(\bx')  \} \d \bx \d \bx'\;.
\end{align}
We now claim that this determinant is close to one because $\bV$ and $\bU$ 
are nearly-orthogonal, and thus the singular values of the matrix $\vec U{{\vec R_\vec V}}_2^\top$ 
are all close to one.  The requirement $d> m^C$ below holds by assumption.

\begin{restatable}{claim}{DETERMINANT} \label{cl:bound_determinant}
\new{If $d> m^C$ for a sufficiently large absolute constant $C$, then}
$\det(\vec U{{\vec R_\vec V}}_2^\top) \geq 1/2$.
\end{restatable}
\begin{proof}
To prove this claim, we show that \new{the singular values} 
of the matrix $\vec U{{\vec R_\vec V}}_2^\top$ are close to $1$. 
Recall that we have assumed that $\bU\bU^\top = \bV\bV^\top= \bI_d$ 
and $\|\vec U\vec V^\top\|_\fr\lesssim  d^{-1/10}$.
    We have that
    \begin{align*}
        m&=\|\vec U\|_\fr^2 =\|\vec U{{\vec R_\vec V}}^\top\|_\fr^2
        \leq \|\vec U{{\vec R_\vec V}}_1^\top\|_\fr^2 +\|\vec U{{\vec R_\vec V}}_2^\top\|_\fr^2
        \\&= \|\vec U\vec V^\top\|_\fr^2 +\|\vec U{{\vec R_\vec V}}_2^\top\|_\fr^2
         \leq C d^{-1/5}+\|\vec U{{\vec R_\vec V}}_2^\top\|_\fr^2\;,
    \end{align*}
    where $C$ is some absolute positive constant. Hence, we have that 
    $\|\vec U{{\vec R_\vec V}}_2^\top\|_\fr^2 \geq m- C d^{-1/5}$. Moreover, we also have that 
    $\|\vec U{{\vec R_\vec V}}_2^\top\|_\op\leq 1$, \new{which means that the maximum singular 
    value of $\vec U{{\vec R_\vec V}}_2^\top$ is at most $1$}. Assume that the minimum singular value 
    of $\vec U{{\vec R_\vec V}}_2^\top$ is $\sigma_{\mathrm{min}}$. Then, we have that
    \[
m-1+ \sigma_{\mathrm{min}}^2 \geq\|\vec U{{\vec R_\vec V}}_2^\top\|_\fr^2 \geq m- C d^{-1/5}\;.
    \]
    Hence, $\sigma_{\mathrm{min}}^2 \geq 1-Cd^{-1/5}$, and therefore
    all the singular values of $\vec U{{\vec R_\vec V}}_2^\top$ are at least $(1-Cd^{-1/5})^{1/2}$.  
    Therefore, we have
    $\det(\vec U{{\vec R_\vec V}}_2^\top) \geq (1-C  d^{-1/5})^{m/2}\geq 1-C (m/2)d^{-1/5}\geq 1/2$ 
    \new{for $d>(Cm)^5$} (which is true by assumption). 
    This completes the proof of \Cref{cl:bound_determinant}.
\end{proof}

 We are now ready to further bound our integral $\mathcal{I}_{\vec V,\vec U}$. First, by writing the distribution $A$ as a mixture $\sum_{ i \in [k]}\lambda_i A_i(\bx)$, we can break $\mathcal{I}_{\vec V,\vec U}$ into contributions from every pair of components. We have the following series of inequalities (see below for step-by-step explanations):
\begin{align}
     \mathcal{I}_{\vec V,\vec U}
    &\lesssim \iint_{\bx, \bx' \in \R^m} \min \{ A(\bx) , A(\bx') \} \d \bx \d \bx' \notag \\
    &=\iint_{\bx, \bx' \in \R^m} \min \left\{ \sum_{i \in [k]} \lambda_i A_i(\bx) , \sum_{j \in [k]} \lambda_j A_j(\bx') \right\} \d \bx \d \bx' \notag \\
    &\leq \sum_{i,j \in [k]} \iint_{\bx, \bx' \in \R^m} \min \{ \lambda_i A_i(\bx) ,  \lambda_j A_j(\bx') \} \d \bx \d \bx'  \label{eq:stepp1} \\
    &\leq \sum_{i,j \in [k]} \iint_{\bx, \bx' \in \R^m} \max\{\lambda_i,\lambda_j \} \min \{  A_i(\bx) ,  A_j(\bx') \} \d \bx \d \bx' \notag \\
    &\leq \sum_{i,j \in [k]} \iint_{\bx, \bx' \in \R^m}\hspace{-30pt} \lambda_i \min \{  A_i(\bx) ,  A_j(\bx') \} \d \bx \d \bx' + \hspace{-5pt} \sum_{i,j \in [k]} \iint_{\bx, \bx' \in \R^m}\hspace{-30pt} \lambda_j \min \{  A_i(\bx) ,  A_j(\bx') \} \d \bx \d \bx'  \label{eq:stepp2} \\
    &= k \hspace{-5pt} \sum_{i,j \in [k]} \iint_{\bx, \bx' \in \R^m} \hspace{-30pt}(\lambda_i/k) \min \{  A_i(\bx) ,  A_j(\bx') \} \d \bx \d \bx'\hspace{-4pt} + \hspace{-7pt}\sum_{i,j \in [k]} \iint_{\bx, \bx' \in \R^m}  \hspace{-30pt}(\lambda_j/k) \min \{  A_i(\bx) ,  A_j(\bx') \} \d \bx \d \bx'  \notag \\
    &\leq 2 k \max_{i,j \in [k]} \iint_{\bx, \bx' \in \R^m} \min \{  A_i(\bx) ,  A_j(\bx') \}\d \bx \d \bx' \;, \label{eq:interm1}
 \end{align} 
 where \Cref{eq:stepp1} uses that $\min(a+b,c) \leq \min(a,c) + \min(b,c)$, \Cref{eq:stepp2} 
 uses that $\max(a,b) \leq a + b$, and 
 for the last step one can view the double summation in the first term 
 of the penultimate line as an expectation over the random choice of the indices $i,j$ 
 according to the distribution that selects $j$ uniformly at random from $[k]$ 
 and makes $i$ equal to $\ell$ with probability $\lambda_\ell$. 
 A similar argument can be used for the second term of the penultimate line. 
 Since the expectation is always smaller than the maximum value, the last line follows.

Recall that each component $A_i$ of the mixture distribution $A$ 
is by definition Gaussian with variance \new{$\delta := c k^{-2.5/m}$} in all directions. 
Let \new{$R := C' \sqrt{\delta m \log(1/\delta)}$} for a sufficiently large constant $C'$ 
so that: $\Pr_{\bz \sim \cN(0,2\delta \bI_m)}[\|\bz\|_2 > R] \leq \delta$. 
This can be seen as follows:
 \begin{align}
    \Pr_{\bz \sim \cN(0,2\delta \bI_m)}[\|\bz\|_2 > R] 
    &= \Pr_{\bz \sim \cN(0,2\delta \bI_m)}[\|\bz\|_2 > C' \sqrt{ \delta m \log(1/\delta)}] \notag \\
&\leq \Pr_{\bz \sim \cN(0,2\delta \bI_m)}[\|\bz\|_2- \sqrt{\delta m} > (C'/2)\sqrt{ \delta \log(1/\delta)}] \label{eq:step2} \\
    &\leq 2\exp\left( - \frac{(C'/2)^2 \delta \log(1/\delta)}{ 32\delta}  \right) 
     \leq \delta\;,  \label{eq:gaussian-conc}
\end{align}
where \Cref{eq:step2} uses the fact that 
$C' \sqrt{ \delta m \log(1/\delta)} -\sqrt{\delta m} = \sqrt{\delta m}(C' \sqrt{\log(1/\delta)}-1) \geq (C'/2)\sqrt{ \delta m \log(1/\delta)} \geq (C'/2)\sqrt{ \delta \log(1/\delta)}$ 
with the penultimate step being true because $C'$ large enough and $\delta<0.1$.  
The last step in \Cref{eq:gaussian-conc} uses \Cref{lem:norm-conc} 
with \new{$\beta = (C'/2) \sqrt{ \delta \log(1/\delta)}$.}

We can thus break the integral appearing in \Cref{eq:interm1} into parts 
based on whether $\bx$ and $\bx'$ fall within or outside a ball of radius $R$
around the mean of the component (recall that $R$ is the radius used in \Cref{eq:gaussian-conc}). 
For each individual integral, we will use \Cref{eq:gaussian-conc} to bound the mass 
of the distribution outside of the ball and bound the mass inside the ball by the volume of that ball. 
Then, by bounding above that volume and after some algebra, 
we can bound all terms by the following (the calculations are deferred to \Cref{sec:appendix_dtv}):

 \begin{restatable}{claim}{CASEANALYSIS}
 $\mathcal{I}_{\vec V,\vec U} \leq C^{m} k \delta^{0.4 m}$ for a sufficiently large absolute constant $C$.
 \end{restatable}
The total variation distance is thus 
$\dtv(P_{A,\bU}, P_{A,\bV}) =  1 - \int_{z \in \R^d} \min \{ P_{A,\vec V}(\bz),  P_{A,\vec V}(\bz)\} \d \bz \geq 1-C^m k \delta^{0.4 m} \geq \new{0.99}$,
where the last step uses that $\delta = ck^{-2.5/m}$ for an appropriately small constant $c>0$. 
The remaining part of the claim that $\dtv(P_{A,\bV}, \cN(\vec 0,\bI_d)) > \new{0.99}$ 
can be handled with similar arguments, and is deferred to \Cref{claim:tv-vound} in the Appendix. 

\medskip

\noindent\textbf{Proof of \Cref{it:chi-square}}
We first focus on a single component $A_i$, which is a spherical Gaussian with mean 
$\bm \mu_i = (\mu_{i,1}, \ldots, \mu_{i,m})$ and variance $\delta$ \new{in each direction}. 
Because both $A_i$ and the standard Gaussian are product distributions in $m$ dimensions, 
the integral in the definition of the $\chi^2(A_i,\cN(\vec 0,\bI_m)$ is separable and we can use  \Cref{fact:chi_square_for_gaussians3} for each coordinate. Concretely, let $\phi$ denote the pdf of $\cN( 0,1)$:
\begin{align*}
        1+\chi^2(A_i,\cN(0,\bI_m)) &= \int_{\bx \in \R^m} \frac{A_i^2(\bx)}{\phi(x_1) \cdots \phi(x_m)} \d \bx 
        =  \prod_{j=1}^m \int_{x_j \in \R} \frac{  \frac{1}{2\pi \delta} \exp\left({-\frac{(x_j - \mu_{i,j})^2}{\delta}}\right)      }{\phi(x_j)} \d x_j\\
        &= \prod_{j=1}^m  (1+\chi^2(\cN(\mu_{i,j},\delta),\cN(0,1))
        = \frac{1}{(\delta(2-\delta))^{m/2}}\exp\left( \frac{\| \bm \mu_i  \|_2^2}{2-\delta} \right) \\
        &\leq \delta^{-m/2} e^{1.21m} \;,
    \end{align*}
    where the last line uses that $\delta<1$ and $\| \bm \mu_i  \|_2 \leq 1.1\sqrt{m}$  by \Cref{it:norm} of \Cref{prop:hard_instance}.
    Denote by $w_i$ the weights in the mixture $A = \sum_{i=1}^k w_i A_i$. Also, by using $\phi_m(\bx)$ to denote the pdf of $\cN(\vec 0,\bI_m)$ we have that
\begin{align*}
        1+\chi^2(A,\cN(\vec 0,\bI_m)) &= \sum_{i=1}^k \sum_{j=1}^k w_i w_j \int_{\bx \in \R^m} \frac{A_i(\bx)A_j(\bx)}{\phi_m(\bx)} \d\bx\\
        &\leq \sum_{i=1}^k \sum_{j=1}^k w_i w_j \sqrt{\int_{\bx \in \R^m} \frac{A_i(\bx)^2 }{\phi(\bx)} \d\bx \int_{\bx \in \R^m} \frac{A_j(\bx)^2 }{\phi_m(\bx)} \d\bx  } \\
        &= \sum_{i=1}^k \sum_{j=1}^k w_i w_j \sqrt{\left(1+\chi^2(A_i,\cN(0,\bI_m)) \right)\left(1+\chi^2(A_j,\cN(0,\bI_m)) \right)} \\
        &\leq \delta^{-m/2} e^{1.21m} \sum_{i=1}^k \sum_{j=1}^k w_i w_j =  \delta^{-m/2} e^{1.21m}  \;,
    \end{align*}
    where the second line uses the Cauchy-Schwartz inequality, and the last line uses the upper bound for $1+\chi^2(A_i,\cN(\vec 0,\bI_m)) $ that we showed in the beginning.
    This completes the proof.

\section{\new{Beating Separation of $\Omega(\sqrt{k})$ : Proof of Theorem~\ref{thm:sqrt-k-informal}}}\label{sec:sqrt-separation}

\new{In this section, we prove the following result which is the formal version of 
Theorem~\ref{thm:sqrt-k-informal}.}

\begin{theorem}[\new{Quadratic SQ Lower Bound for Separation $\sim k^{1/2}$}]\label{thm:epsilon_half}
Let $C>0$ be a sufficiently large absolute constant. 
Let  $d,k \in \Z_+$ and $c\in (0,2/9)$ with \new{$d>(1/c)^{C/c}$}, \new{$2 \leq k  \leq (c/C)\log d$ }. 
\new{Consider} the following hypothesis testing problem regarding a distribution $P$ on $\R^d$: 
    \begin{itemize}[leftmargin=*]
        \item(Null Hypothesis) $P = \cN(\vec 0, \bI_d)$.
        \item(Alternative Hypothesis)  $P$ belongs to a family $\cP$, every member of which is a mixture 
        of Gaussians $\sum_{i=1}^k w_i \cN(\bm \mu_i, \vec \Sigma)$ \new{ with uniform weights $w_i=1/k$},  
        mean vectors with pairwise separation $\|\bm \mu_i-\bm \mu_j\|_2 \geq \new{\sqrt{k}/3}$
         for all $i \neq j \in [k]$, and common covariance matrix $\vec \Sigma \preceq \bI_d$. 
         Moreover, $\dtv(P,\cN(\vec 0,\bI_d))>0.99$ and $\dtv(P,P')>0.99$ for all distinct $P,P' \in \cP$.
    \end{itemize}
     Any algorithm with statistical query access to $P$ that distinguishes correctly between the two cases, does one of the following: it performs $2^{\Omega(d^{2c})}$ statistical queries, 
     or it uses at least \new{one statistical query to $\mathrm{VSTAT}(\Omega(d^{2-9c}))$.}
\end{theorem}

\new{We start with a brief overview of the new ideas required for the proof.}

\new{First, it is instructive to explain why \Cref{thm:main} and its proof do not suffice for our purposes.}
In particular, to use \Cref{thm:main} in order to obtain an SQ lower bound 
of $2^{d^{\Omega(1)}}$ queries vs a query to $\mathrm{VSTAT}(d^2)$, 
we need to set the parameter $\eps$ \new{(where the separation is $\Delta = k^{\eps}$)} sufficiently small. 
This is because in that theorem, $\eps$ appears inside a big-$\Omega$ notation in the query tolerance 
and a closer examination of our proofs reveals that the hidden constant in that big-$\Omega$ is rather large 
(in the order of hundreds). Thus, \Cref{thm:main} cannot yield a super-linear SQ lower bound 
for the $\eps=1/2$ case, which corresponds to pairwise separation of $\sim \sqrt{k}$.

\new{In more detail,} the constant factor in front of $\eps$ in \Cref{thm:main} is large for two reasons: 
(i) The number of Gaussian components in our construction (c.f.\ \Cref{prop:hard_instance}) 
was $k^{26 \eps t}$, meaning that we had to match $t=1/(26\eps)$ many moments 
in order to end-up with  $k$ components, 
and (ii) the fact about random matrices being nearly orthogonal (\Cref{fact:setofmatrices}) 
that we used was suboptimal. In particular, while the corresponding fact 
for vectors states that any pair of random unit vectors has inner product very close to $O(d^{-1/2})$, 
the generalization of that to matrices by \Cref{fact:setofmatrices} stated 
that the pairs of random matrices $\bU,\bV$ have $\|\bU \bV^\top \|_\fr \leq O(d^{-1/10})$. 
The constant in the exponent is crucial \new{here} because it also appears in front of $\eps$ 
in the final SQ lower bound.

In this section, we \new{overcome both} of these issues 
by providing a tighter \new{construction and} analysis for the $\eps=1/2$ case. 
In particular, we replace the existential LP-duality argument of  \Cref{prop:hard_instance}  
by a simpler constructive proof (cf.\ \Cref{lem:first_three_moments});  
\Cref{lem:first_three_moments} provides a discrete distribution matching the first three moments 
with the standard Gaussian. Moreover, we provide a tight version of \Cref{fact:setofmatrices} 
via an improved analysis (\Cref{fact:setofmatrices-improved}).

\new{With these tools,} we are able to show that any SQ algorithm 
for distinguishing between $\cN(\vec 0, \bI)$ and a $k$-GMM with unknown bounded covariance 
and mean separation of the order of $\sqrt{k}$ has nearly quadratic complexity.

We now establish the existence of a simple discrete distribution 
that matches its first three moments with the standard Gaussian.

\begin{lemma}[Moment Matching]\label{lem:first_three_moments}
There exists a discrete distribution $D$ on $\R^m$ such that: 
(i) $D$ is supported on $2m$ points, (ii) $D$  matches the first three moments with $\cN(\vec 0 ,\vec I_m)$, 
and (iii) for every pair of distinct points $\bx,\by$ in the support of $D$, it holds $\|\bx-\by\|_2 \geq \sqrt{m}$.
\end{lemma}
\begin{proof}
Let $\vec e_i$ for $i \in [m]$ denote the $i$-th vector of the standard basis of $\R^m$, 
i.e., the vector having $1$ in the $i$-th coordinate and zero everywhere else.
Let the set of vectors $S = \{\vec x_1,\ldots, \vec x_{2m} \}$ defined as $\vec x_i = \sqrt{m/2}\, \vec e_i$ 
for $i=1,\ldots,m$, and $\vec x_i = -\sqrt{m/2}\, \vec e_{i-m}$ for $i=m+1,\ldots,2m$.
    
It is easy to verify that $D = \cU(S)$, the uniform distribution on these points, 
matches the first three moments with $\cN(\vec 0 ,\vec I_m)$: 
Let $p$ be a polynomial of degree at most $3$, i.e., $p(x_1,\ldots,x_m)=x_1^a x_2^b x_3^c$, 
with $a+b+c \leq 3$ (without loss of generality, we assumed that the coordinates from $[m]$ 
with non-zero power are the first three). If either of $a,b,c$ is equal to $1$ or $3$, 
then $\E_{\bx \sim \cU(S)}[p(\bx)] = 0$, because we made $S$ symmetric about the origin. 
This only leaves the case $p(x_1,\ldots,x_m) = x_1^2$, 
where we have $\E_{\bx \sim \cU(S)}[p(\bx)] = 1$, 
because the first coordinate is equal to $\sqrt{m/2}$ and $-\sqrt{m/2}$ only for two points in $S$ 
and zero for every other one. This completes the proof. 
\end{proof}

We provide the tightening of \Cref{fact:setofmatrices} in the lemma below. 
The proof is deferred to \Cref{sec:appendix-linear-algebra}.

\begin{restatable}{lemma}{MATRICESORTH} \label{fact:setofmatrices-improved}
Let $C$ be a sufficiently large absolute constant. 
Let $c \in (0,1/4)$ and $m,d \in \N$ with \new{$d>(1/c)^{C/c}$ and $m<d^{c/5}/C$}. 
There exists a set $S$ of $2^{\Omega(d^{2c})}$ matrices in $\R^{m \times d}$ such that 
every $\bA \in S$ satisfies $\bA \bA^\top = \bI_m$ and every pair $\vec A,\vec A' \in S$ 
with $\vec A \neq \vec A'$ satisfies $\|\bA' \bA^\top \|_\op \lesssim d^{-1/2+2c}$.
\end{restatable}

\new{We can now give the proof of the main result of this section.}

\begin{proof}[Proof of \Cref{thm:epsilon_half}]
Let $C$ be a sufficiently large constant.
Let $D$ be the distribution from \Cref{lem:first_three_moments} with $m:=k/2$ 
and $A = U_\rho D$ for $\delta=k^{-2.5/m}/C$, where $U_\rho$ denotes 
the Ornstein-Uhlenbeck operator with parameter $\rho$. 
We choose $\rho = \sqrt{1-\delta}$.

The above means that $A$ is a mixture of $k$ equally weighted spherical Gaussians in $\R^m$, 
each with variance $\delta$ in every direction. 
\new{By \Cref{lem:first_three_moments}, the mean separation is 
$\rho\cdot \sqrt{k/2} = \sqrt{1-k^{-2.5/k}/C}\sqrt{k/2} \geq \sqrt{k}/3$ for any $k\geq 2$.}

The following can be shown by repeating mutatis-mutandis 
the same steps we followed while proving \Cref{prop:hard_instance}: 
\begin{enumerate}[leftmargin=*]
    \item The first $3$ moments of $A$ match with those of $\cN(\vec 0, \bI_m)$. 
    \item For every $\bU,\bV \in \R^{m \times d}$ with $\bU\bU^\top = \bV\bV^\top= \bI_d$  
    and $\|\vec U \vec V^\top \|_\fr = O(d^{-1/2+2c})$, it holds $\dtv(P_{A,\bU}, P_{A,\bV}) > 0.99$. 
    Moreover, for all  $\bV\in \R^{m \times d}$ it holds $\dtv(P_{A,\bV}, \cN(0,\bI_d)) > 0.99$.\looseness=-1\label{it:dtv-separation}
    \item $\chi^2(A,\cN(\vec 0,\vec I_m) \leq e^{O(k)}$.
\end{enumerate}
Now by also following \new{the same} steps \new{as} in the proof of \Cref{thm:main}, 
but replacing \Cref{fact:setofmatrices} by \Cref{fact:setofmatrices-improved}, 
we obtain that every SQ algorithm for solving our hypothesis testing problem, 
either needs $2^{\Omega(d^{2c})}$ queries or at least one query to
\begin{align*}
   \mathrm{VSTAT}( \Omega(d^{2-8c})/ \chi^2(A,\normal(\vec 0,\vec I_m)) ) \;.
\end{align*}
We note that $ \Omega(d^{2-8c}) /\chi^2(A,\normal(\vec 0,\vec I_m))  \geq  \Omega(d^{2-8c})e^{-O(k)} \geq \Omega(d^{2-9c})$, 
where the last inequality uses our assumption \new{$k<(c/C)\log d$}. 
Also note that \Cref{fact:setofmatrices-improved} was indeed applicable, 
since its \new{requirement $m<d^{c/5}/C$ is satisfied 
because we have  $m:=k/2 < 0.5 (c/C)\log d < d^{c/5}/C$, 
where the first inequality is one of our assumptions and the second follows by our other assumption $d>(1/c)^{C/c}$.}
This completes the proof of \Cref{thm:epsilon_half}.
\end{proof}

\clearpage

\bibliographystyle{alpha}
\bibliography{allrefs}

\newpage

\appendix

\section{Additional Preliminaries} \label{sec:additional_prelim}

\subsection{Additional Notation}
 We use $\Z$ for the set of integers and $\Z_+$ for positive integers. For $n \in \Z_+$, we denote $[n] \eqdef \{1,\ldots,n\}$ and use $\cS^{d-1}$ for the $d$-dimensional unit sphere. We use $\mathcal{S}_{d-1}(R)$  to denote the $d$ dimensional sphere with radius $R$ and center the origin.
For a vector $\bv$, we let $\|\bv\|_2$ denote its $\ell_2$-norm. 
We use $\bI_d$ to denote the $d \times d$ identity matrix. We will drop the subscript when it is clear from the context.
For a matrix $\vec A$, we use $\|\vec A\|_\fr$ and $\|\vec A\|_{\op}$ to denote the Frobenius and spectral (or operator) norms respectively. If $\vec a=(a_1,\ldots, a_m) \in \Z_+^m$ is a multi-index, we denote  $|\vec a | = \sum_{i=1}^m a_i$

We use $a\lesssim b$ to denote that there exists an absolute universal constant $C>0$ (independent of the variables or parameters on which $a$ and $b$ depend) such that $a \leq C  b$.
  
We use the notation $x \sim D$ to denote that a random variable $x$ is distributed according to the distribution $D$. For a random variable $x$, we use $\E[x]$ for its expectation.	
	 We use $\cN(\bm{\mu}, \vec \Sigma)$ to denote the Gaussian distribution with mean $\bm \mu$ and covariance matrix $\vec \Sigma$. For a set $S$, we use $\cU(S)$ to denote the uniform distribution on $S$ and use $x \sim S$ as a shortcut for $x \sim \cU(S)$. We denote by $\phi_m(\bx)$ the probability density function (pdf) of the standard Gaussian in $m$-dimensions $\cN(\vec 0, \bI_m)$, and by $\phi(x)$ the pdf of the univariate standard Gaussian $\cN(0,1)$. We slightly abuse notation by using the same letter for a distribution and its pdf, e.g., we will denote by $P(\bx)$ the pdf of a distribution $P$. We use $\dtv(P,Q)$ for the total variation distance between two distributions $P,Q$.

We will prefer to use capital letters for constants that are assumed to be sufficiently large and small letters for constants that need to be sufficiently small.

\subsection{Hermite Analysis}  
Hermite polynomials form a complete orthogonal basis of the vector space $L_2(\R,\cN(0,1))$ of all functions $f:\R \to \R$ such that $\E_{x\sim \cN(0,1)}[f^2(x)]< \infty$. There are two commonly used types of Hermite polynomials. The \emph{physicist's } Hermite polynomials, denoted by $H_k$ for $k\in \Z$ satisfy the following orthogonality property with respect to the weight function $e^{-x^2}$: for all $k,m \in \Z$, $\int_\R H_k(x) H_m(x) e^{-x^2} \d x = \sqrt{\pi} 2^k k! \mathbf{1}(k=m)$. The \emph{probabilist's} Hermite polynomials $H_{e_k}$ for $k\in \Z$ satisfy $\int_\R H_{e_k}(x) H_{e_m}(x) e^{-x^2/2} \d x = k! \sqrt{2\pi}  \mathbf{1}(k=m)$ and are related to the physicist's polynomials through $H_{e_k}(x)=2^{-k/2}H_k(x/\sqrt{2})$. 
	We will mostly use the \emph{normalized probabilist's} Hermite polynomials $h_k(x) = H_{e_k}(x)/\sqrt{k!}$, $k\in \Z$ for which $\int_\R h_k(x) h_{m}(x) e^{-x^2/2} \d x = \sqrt{2\pi} \mathbf{1}(k=m)$.
	These polynomials are the ones obtained by Gram-Schmidt orthonormalization of the basis $\{1,x,x^2,\ldots\}$ with respect to the inner product $\langle{f},{g}\rangle_{\cN(0,1)}=\E_{x \sim \cN(0,1)}[f(x)g(x)]$. Every function $f \in L_2(\R,\cN(0,1))$ can be uniquely written as $f(x) = \sum_{i \in \Z} a_i h_i(x)$ and we have $\lim_{n \rightarrow \infty}\E_{x \sim \cN(0,1)}[(f(x)- \sum_{i =0}^n a_i h_i(x))^2] = 0$ (see, e.g., \cite{andrews_askey_roy_1999}).
Moreover, we have the following explicit expression of $h_i(\cdot)$ (see, for example, \cite{andrews_askey_roy_1999,Szego:39}):
\begin{align}
    h_i(x) &= \sqrt{i!} \sum_{j=0}^{\lfloor i/2 \rfloor } \frac{(-1)^j}{j!(i - 2j)!} \frac{x^{i-2j}}{2^j} \;. \label{eq:explicit-hermite}
\end{align}	
Extending the normalized probabilist's Hermite polynomials to higher dimensions, an orthonormal basis of $L_2(\R^d,\cN(\vec 0,\vec I_d))$ (with respect to the inner product 
$\langle f, g\rangle = \E_{\bx \sim \cN(\vec 0, \vec I_d)}[f(\bx)g(\bx)]$)  
can be formed by all the products of one-dimensional Hermite polynomials, i.e., $h_{\vec a}(\bx) = \prod_{i=1}^d h_{a_i}(x_i)$, 
for all multi-indices $\vec a \in \Z^d$ (we are now slightly overloading notation by using multi-indices as subscripts). The total degree of $h_\vec a$ is $|\vec a|=\sum_{i=1}^d a_i$.

\paragraph{Ornstein-Uhlenbeck Operator}
For a $\rho > 0$, we define the \emph{Gaussian noise} (or \emph{Ornstein-Uhlenbeck}) operator $U_\rho$ as the operator that maps a distribution $F$ on $\R^m$ to the distribution of the random variable $\rho \bx + \sqrt{1-\rho^2}\bz$, where $\bx \sim F$ and $\bz \sim \cN(\vec 0,\vec I_m)$ independently of $\bx$. 
A standard property of the $U_\rho$ operator is that it operates diagonally with respect to Hermite polynomials:\looseness=-1

\begin{fact}[see, e.g., Proposition 11.37 in~\cite{AoBF14}] \label{fact:eigenfunction} 
    For any multivariate Hermite polynomial $h_{\vec a}$, any  $F$ on $\R$, and $\rho\in(0,1)$, that $\E_{\bx \sim U_\rho F}[h_{\vec a}(\bx)] = \rho^{|\vec a|} \E_{\vec x \sim F}[h_{\vec a}(\bx)]$, where $|\vec a|=\sum_{i} a_i$.
\end{fact}

\subsection{Background on the Statistical Query Model}\label{sec:sq-background-appendix}

\begin{definition}[Decision Problem over Distributions] \label{def:decision}
Let $D$ be a fixed distribution and $\D$ be a distribution family.
We denote by $\mathcal{B}(\D, D)$ the decision (or hypothesis testing) problem
in which the input distribution $D'$ is promised to satisfy either
(a) $D' = D$ or (b) $D' \in \D$, and the goal
is to distinguish between the two cases.
\end{definition}

\begin{definition}[Pairwise Correlation] \label{def:pc}
The pairwise correlation of two distributions with probability density functions
$D_1, D_2 : \R^d \to \R_+$ with respect to a distribution with
density $D: \R^d \to \R_+$, where the support of $D$ contains
the supports of $D_1$ and $D_2$, is defined as
$\chi_{D}(D_1, D_2) = \int_{\R^d} D_1(\bx) D_2(\x)/D(\bx)\, \d\bx - 1$.
\end{definition}

\begin{definition} \label{def:uncor}
We say that a set of $s$ distributions $\mathcal{D} = \{D_1, \ldots , D_s \}$
 is $(\gamma, \beta)$-correlated relative to a distribution $D$
if $|\chi_D(D_i, D_j)| \leq \gamma$ for all $i \neq j$,
and $|\chi_D(D_i, D_j)| \leq \beta$ for $i=j$.
\end{definition}

\subsection{Miscellaneous Facts}

We require the standard concentration of the norm of Gaussian vectors (see, e.g., Theorem 3.1.1 of \cite{Ver18} or Theorem 4.7 of \cite{hdpnotesWegner}):
\begin{fact}[Gaussian Norm Concentration] \label{lem:norm-conc} 
    For every $0\leq \beta \leq \sigma \sqrt{d}$ we have that
    \begin{align*}
        \pr_{\vec x \sim \cN(\vec 0,\sigma^2 \bI_d)}[|\|\bx\|_2- \sigma \sqrt{d}| > \beta] \leq
        2\exp\left(-\frac{\beta^2}{16\sigma^2}   \right) \;.
    \end{align*}
\end{fact}

We also require the following result stating the random Gaussian vectors are nearly-orthogonal.

    \begin{fact} [\cite{CaiFanJiang}, also see Corollary D.3 in \cite{DKS17-sq}] \label{fact:orthogonality}
        Let $\theta$ be the angle between two random unit vectors uniformly distributed  over $\cS^{d-1}$. Then, we have that
            $\pr[|\cos \theta | \geq d^{-\alpha}] \leq e^{- d^{1-2\alpha}/5}$,
        for any $0 \leq\alpha  \leq 1/2$.
    \end{fact}
    
\begin{fact}[Gaussian Hypercontractivity \cite{Bog:98,nelson1973free}] \label{lem:hypercontractivity}
If $p: \R^m \to \R$ is a polynomial of degree at most $k$, for every $t\geq 2$, 
\begin{align*}
    \E_{\bx \sim \cN(\vec 0,\vec I_m)}\left[|p(\bx)|^t\right]^{\frac{1}{t}} \leq (t-1)^{k/2}\sqrt{\E_{\bx \sim \cN(\vec  0,\vec I_m)}\left[p^2(\bx) \right]} \;.
\end{align*}
\end{fact}

\begin{fact}[Volume of $d$-Ball] \label{fact:volume}
For any $R>0$ let $\mathcal{S}_{d-1}(R) = \{\x \in \R^d: \|\x\|_2\leq R  \}$. Then,
\begin{align*}
    \mathrm{Vol}(\mathcal{S}_{d-1})  = O \left( \frac{1}{\sqrt{\pi d}} \left( \frac{2\pi e}{d} \right)^{d/2} R^d \right) \;.
\end{align*}
\end{fact}

	\begin{fact}\label{fact:chi_square_for_gaussians3}
	The following holds for the chi-square divergence between two univariate Gaussians:
	\begin{align*}
	   \chi^2(\cN(\mu_1,\sigma_1^2),\cN(\mu_2,\sigma_2^2)) = \frac{\sigma_2^2}{\sigma_1\sqrt{2 \sigma_2^2-\sigma_1^2}}\exp\left(\frac{(\mu_1-\mu_2)^2}{2\sigma_2^2 - \sigma_1^2}  \right) - 1 \;.
	\end{align*}
	\end{fact}

In the following we let $\mathbb{C}$ denote the set of complex numbers.

\begin{definition}[Gershgorin Discs]
    For any complex $n \times n$ matrix $\vec A$, for $i \in [n]$, let $R'_i(\vec A) = \sum_{j \neq i}|a_{ij}|$ and let $G(\bA) = \bigcup_{i=1}^n \{z \in \mathbb{C}:|z-a_{ii}|\leq R_i'(\bA) \}$. Each disc $\{z \in \mathbb{C}:|z-a_{ii}|\leq R_i'(\bA) \}$ is called Gershgorin disc and their union $G(\bA)$ is called the Gershgorin domain.
\end{definition}

\begin{fact}[Gershgorin's Disc Theorem]\label{fact:Gershgorin}
For any complex $n\times n$ matrix $\bA$, all the eigenvalues of $\bA$ belong to the Gershgorin domain $G(\bA)$.
\end{fact}

\section{Omitted Proofs from \Cref{sec:hard-instance}}\label{sec:appendix-sec4}

\subsection{Concentration of Gaussian Polynomials}\label{sec:concentration}

We restate and prove the following:

\MOMENTS*

The proof follows by applying the lemma below for the polynomials $p,p^2$ and $p^4$ which are of degree $k=t,2t$ and $4t$ respectively.

\begin{lemma} \label{lem:polys-conc}
For any $\eps>0$, if a set $S$ of $N > 10 \sigma^2 \new{m^{3k}}/\eps^2$ samples is drawn i.i.d.\ from $\cN(\vec 0,\vec I_m)$, then with probability at least $0.9$ we have that for all polynomials $p : \R^m \to \R$ with $\E_{\bx \sim \cN(\vec 0,\vec I_m)}[p^2(\bx)] \leq \sigma^2$ and degree at most $k$  it holds that
\begin{align*}
    \left| \E_{\x\sim \cU(S)}[p(\bx)] - \E_{\bx \sim \cN(\vec 0,\vec I_m)}[p(\bx)] \right| \leq \eps \;.
\end{align*}
\end{lemma}
\begin{proof}
First, using Chebyshev's inequality, we have the following concentration for every normalized probabilist's Hermite polynomial:
    \begin{align}
         \pr_{\bx_1,...,\bx_N \sim \cN(\vec 0, \vec I_m) }\left[   \left| \E_{\x\sim \cU(S)}[h_{\vec J}(\bx)] - \E_{\bx \sim \cN(\vec 0, \vec I_m)}[h_{\vec J}(\bx)]  \right| > \frac{\eps}{m^k\sigma} \right] &\leq \frac{\sigma^2 m^{2k}}{N \eps^2} \Var_{\bx \sim \cN(\vec 0,\vec I_m)}[h_{\vec J}(\bx)] \notag\\
    &= \frac{\sigma^2 m^{2k}}{N \eps^2} \E_{\bx \sim \cN(\vec 0,\vec I_m)}[h^2_{\vec J}(\bx)] \notag\\
    &= \frac{\sigma^2 m^{2k}}{N \eps^2}  \leq \frac{0.1}{m^k} \;, \label{eq:chebysevv}
    \end{align}
where the last line used that $N > 10 \sigma^2 m^{3k}/\eps^2$. In what follows we condition on the event that $|\E_{\x\sim \cU(S)}[h_{\vec J}(\bx)] - \E_{\bx \sim \cN(\vec 0, \vec I_m)}[h_{\vec J}(\bx)]| \leq \eps$ for all ${\vec J} \in \N^m : |{\vec J}| \leq k$, which, by a union bound and \Cref{eq:chebysevv} holds with probability at least $0.9$.
We expand $p(\bx)$ on the basis of the normalized probabilist's Hermite polynomials $p(\bx) = \sum_{{\vec J} \in \N^m : |{\vec J}| \leq k} a_{\vec J} h_{\vec J}(\bx)$, and note that $|a_{\vec J}| \leq \sigma$ for all these coefficients (because by Parseval's identity $\sum_{{\vec J}} a_{\vec J}^2 \leq \sigma^2$). Therefore, we conclude that
\begin{align*}
     \left| \E_{\x\sim \cU(S)}[p(\bx)] - \E_{\bx \sim \cN(\vec 0,\vec I_m)}[p(\bx)] \right| &\leq \sum_{{\vec J} \in \N^m : |{\vec J}| \leq k} |a_{\vec J}| \left| \E_{\x\sim \cU(S)}[h_{\vec J}(\bx)] - \E_{\bx \sim \cN(\vec 0,\vec I_m)}[h_{\vec J}(\bx)]  \right| \\
     &\leq \sigma m^k \eps/(m^k \sigma) =\eps \;.
\end{align*}
    
\end{proof}

\subsection{Omitted Details from Proof of \Cref{it:dtv-separation-or}} \label{sec:appendix_dtv}

\CASEANALYSIS*
 \begin{proof}
 Let $\iint_{\bx, \bx' \in \R^m} \min \{  A_i(\bx) ,  A_j(\bx') \} = \cI_1 + \cI_2 + \cI_3 + \cI_4$, where 
 \begin{enumerate}
     \item $\cI_1 = \iint_{\text{ $\| \bx - \bm \mu_i \|_2 > R$ and $\| \bx' - \bm \mu_j \|_2 \leq R$} } \min \{ A_i(\bx)  , A_j(\bx') \} \d \bx \d \bx'$,
     \item $\cI_2 = \iint_{\text{ $\| \bx - \bm \mu_i \|_2 \leq R$ and $\| \bx' - \bm \mu_j \|_2 > R$} } \min \{ A_i(\bx)  , A_j(\bx') \} \d \bx \d \bx'$,
     \item $\cI_3 = \iint_{\text{ $\| \bx - \bm \mu_i \|_2 > R$ and $\| \bx' - \bm \mu_j \|_2 > R$} } \min \{ A_i(\bx)  , A_j(\bx') \} \d \bx \d \bx'$,
     \item $\cI_4 = \iint_{\text{ $\| \bx - \bm \mu_i \|_2 \leq R$ and $\| \bx' - \bm \mu_j \|_2 \leq R$} } \min \{ A_i(\bx)  , A_j(\bx') \} \d \bx \d \bx'$.
 \end{enumerate}
 We start with the first term. Recall that $A_i$ is an $m$-dimensional Gaussian with mean $\bm \mu_i$ and variance $\delta$ in all directions. We have the following:
 \begin{align}
     \cI_1 &\leq \int_{\| \bx - \bm \mu_i \|_2 > R}  \sqrt{ A_i(\bx)} \d \bx  \int_{\| \bx' - \bm \mu_j \|_2 \leq R} \sqrt{ A_j(\bx')}  \d \bx'  \tag{using $\min(a,b) \leq \sqrt{a b}$}\\ 
     &\leq \int_{\| \bx - \bm \mu_i \|_2 > R}  (2\pi \delta)^{-m/4}e^{ -\frac{\|\bx - \bm \mu_i\|_2^2}{4\delta} } \d \bx   
      \int_{\| \bx' - \bm \mu_j \|_2 \leq R}  (2\pi \delta)^{-m/4}e^{ -\frac{\|\bx' - \bm \mu_j\|_2^2}{4\delta} } \d \bx'  \notag \\
     &\leq (2\pi \delta)^{m/4}  \int_{\| \bx - \mu_i \|_2 > R} (2\pi \delta)^{-m/2}e^{ -\frac{\|\bx - \bm \mu_i\|_2^2}{4\delta}} \d \bx 
      \int_{\| \bx' - \bm \mu_j \|_2 \leq R} (2\pi \delta)^{-m/4}e^{ -\frac{\|\bx' - \bm \mu_j\|_2^2}{4\delta} } \d \bx'  \notag \\
     &\leq (2\pi \delta)^{m/4} \delta \cdot \delta^{-m/4}  \mathrm{Vol}(\mathcal{S}_{d-1}(R)) \tag{using \Cref{eq:gaussian-conc} for the first integral} \\
     &\leq (2\pi)^{m/4} \delta  \left(\frac{1}{\sqrt{ \pi m}} \left( \frac{2\pi e}{m} \right)^{m/2} R^m \right) \tag{by \Cref{fact:volume}}\\
     &\leq C_1^{m}  m^{-m/2}\delta^{1+ m/2} m^{m/2} (\log(1/\delta))^{m/2} \tag{using $R = C' \sqrt{\delta m \log(1/\delta)}$} \\
     &\leq C_1^{m}  \delta^{1+ m/2}  (\log(1/\delta))^{m/2}  \label{eq:firsttermbound}
 \end{align}
 for a sufficiently large constant $C_1$. The same bound can be derived for $\cI_2$. For $\cI_3$ we use \Cref{eq:gaussian-conc} for both integrals to obtain $\cI_3 \leq (2\pi \delta)^{m/2}\delta^2$. Finally, for the last term $\cI_4$ we have that $\cI_4 \leq\delta^{-m/2} (\mathrm{Vol}(\mathcal{S}_{d-1}(R)))^2 \leq  C_2^m  \delta^{-m/2} m^{-m}\delta^{m} m^{m} (\log(1/\delta))^{m} \leq  C_2^m  \delta^{m/2}(\log(1/\delta))^{m} $, where the first step used $\min\{ A_i(\bx), A_j(\bx)\} \leq  \delta^{-m/2}$ and that both integrals are over a ball of radius $R$.
 Putting everything together, we have shown that
 \begin{align}
    \mathcal{I}_{\vec V,\vec U} &\leq   C_3^{m} k \delta^{m/2} (\log(1/\delta))^{m} \leq C_4^{m} k \delta^{0.4 m} \;.
 \end{align}
 \end{proof}
 
 \begin{claim}\label{claim:tv-vound}
 In the setting of \Cref{prop:hard_instance} it holds  $\dtv(P_{A,\bV}, \cN(\vec 0,\bI_d)) > 0.99$.
 \end{claim}
 \begin{proof}
Let $\bv_1,\ldots,\bv_m$ denote the rows of $\bV$ and extend this set to an orthonormal basis $\bv_1,\ldots,\bv_m,\ldots, \bv_d$ of the entire $\R^d$. Let $\bV^\perp$ be the matrix having $\bv_{m+1},\ldots,\bv_d$ as rows and $\vec R$ be the matrix having $\bv_1,\ldots,\bv_m,\ldots, \bv_d$ as rows. Using the definition from \Cref{eq:hiden-dir} (and recalling that $\phi_d(\bx)$ denotes the pdf of $\cN(0,\bI_d$),
    \begin{align*}
         P_{A,\bV(\bz)} = A(\vec V \bz) \phi_{d-m}\left( \mathrm{Proj}_{ \mathcal{V}^\perp}(\bz) \right) 
         = A(\vec V \bz) \phi_{d-m}\left( \bV^\perp \bz \right) \;.
    \end{align*}
    As before, we examine the  integral $\cI := \int_{z \in \R^d} \min \left\{ P_{A,\bV}(\bz), \phi_d(\bz)  \right\}\d \bz $ for which we have the following:
    \begin{align}
        \cI &= \int_{z \in \R^d} \min \left\{ P_{A,\bV}(\bz), \phi_d(\bz)  \right\}\d \bz \notag \\
        &= \int_{z \in \R^d} \min \left\{  A(\vec V \bz) \phi_{d-m}\left( \bV^\perp \bz \right), \phi_{m}\left( \mathrm{Proj}_{ \mathcal{V}}(\bz) \right) \phi_{d-m}\left( \mathrm{Proj}_{ \mathcal{V}^\perp}(\bz) \right) \right\}\d \bz \notag \\
        &= \int_{z \in \R^d} \min \left\{  A(\vec V \bz) \phi_{d-m}\left( \bV^\perp \bz \right), \phi_{m}\left(  \bV  \bz\right) \phi_{d-m}\left(  \bV^\perp \bz \right) \right\}\d \bz \notag \\
        &= \int_{z \in \R^d} \min \left\{  A(\vec V \bR^\top \bz) \phi_{d-m}\left( \bV^\perp  \bR^\top\bz \right), \phi_{m}\left(  \bV   \bR^\top\bz\right) \phi_{d-m}\left(  \bV^\perp  \bR^\top\bz \right) \right\}\d \bz \tag{by rotating space by $ \bR^\top$} \notag \\
        &=\int_{z \in \R^d} \min \left\{  A(z_1,\ldots, z_m) \phi_{d-m}\left( z_{m+1},\ldots, z_d \right), \phi_{m}\left( z_1,\ldots, z_m \right) \phi_{d-m}\left(  z_{m+1},\ldots, z_d  \right) \right\}\d \bz \tag{using the definition of matrices $\bV,\bV^\perp,\bR$} \\
        &= \int_{(z_1,\ldots, z_m) \in \R^m} \min \left\{  A(z_1,\ldots, z_m)  , \phi_{m}\left( z_1,\ldots, z_m \right)   \right\} \d z_1 \cdots \d z_m \notag 
 \\
        &= \int_{\bx \in \R^m} \min \left\{  A(\bx)  , \phi_{m}\left( \bx \right)   \right\} \d \bx \tag{by renaming $\bx = (z_1,\ldots,z_m)$ }  \notag  \\
        &= \int_{\bx \in \R^m} \min \left\{   \sum_{i=1}^k \lambda_i A_i(\bx)  , \phi_{m}\left( \bx \right)   \right\} \d \bx \tag{$A = \sum_{i \in [k]} \lambda_i A_i$}\\
        &\leq  \sum_{i=1}^k \int_{\bx \in \R^m} \min \left\{  \lambda_i A_i(\bx)  , \phi_{m}\left( \bx \right)   \right\} \d \bx \tag{using $\min(a+b,c) \leq \min(a,c) + \min(b,c)$}\\
        &\leq k\max_{i \in [k]} \int_{\bx \in \R^m} \min \left\{   A_i(\bx)  , \phi_{m}\left( \bx \right)   \right\} \d \bx  \;, \label{eq:fistbound}
    \end{align}
    where the last step uses that $\lambda_i\leq 1$. Now, $A_i= \cN( \bm \mu_i,\delta\bI_m)$ with $\|\bm \mu_i\|_2 \geq 0.9 \sqrt{m}$ by \Cref{it:norm} of \Cref{lem:hard-inst} and $\delta$ is smaller than 1, thus we have that $\int_{\bx \in \R^m} \min \left\{   A_i(\bx)  , \phi_{m}\left( \bx \right)   \right\} \d \bx  = 1 - \dtv(\cN(\bm \mu_i,\delta \bI_m), \cN(\vec 0,\bI_m))\leq 1 - \dtv(\cN(\bm \mu_i,\bI_m), \cN(\vec 0,\bI_m))$. By a rotation argument similar to what we did earlier, the contribution comes only from the error along the direction that connects the origin to the point $\bm \mu_i$
    \begin{align*}
        1-\dtv\left(\cN(\bm \mu_i,\bI_m), \cN(\vec 0,\bI_m)\right) &= 1-\dtv\left(\cN(\|\bm \mu_i\|_2,1), \cN(0,1)\right) 
        = \mathrm{erfc}\left( \frac{\|\bm \mu_i\|_2}{2\sqrt{2}} \right) \\
        &\leq \mathrm{erfc}\left( \sqrt{m}/4 \right)
        \leq \frac{1}{100k} \;,   
    \end{align*}
     \new{where the last step requires $m > C \log(k)$, which is true since $m=k^{2\eps}$ and we have assumed $k^{\eps} > C \sqrt{\log k}$}. Putting everything together and combining with the bound of \Cref{eq:fistbound} we conclude that $\dtv(P_{A,\bV}, \cN(0,\bI_d)) = 1-\int_{z \in \R^d} \min \left\{ P_{A,\bV}(\bz), \phi_d(\bz)  \right\}\d \bz  \geq 1- k/(100k) =  0.99$.

\end{proof}

\section{Omitted Proofs from \Cref{sec:sqrt-separation}}\label{sec:appendix-linear-algebra}

We restate and prove the following result.

\MATRICESORTH*

\begin{proof}
    We will use the following basic fact:

    \begin{fact}\label{fact:nearly-orth-vec}
    For any $0<c<1/2$, there exists a set $S'$ of $2^{\Omega(d^{2c})}$ unit vectors in $\R^d$, such that any pair $\vec u,\vec v \in S'$ with $\vec u \neq \vec v$ satisfies $|\vec u^\top \vec v | \lesssim d^{-1/2+c}$.
    \end{fact}

    Let $S' = \{ \vec u_1,\ldots,\vec u_{|S'|} \}$ be the set of vectors from the fact above. Let $S''$ be the set of matrices $\{\vec B_i\}_{i=1}^{ |S'|/m}$ for where $ \vec B_i$ is defined to have as rows the vectors $\vec u_j$ for $j = (i-1)\cdot m +1, \ldots i\cdot m$. \new{Note that $|S'|/m= 2^{\Omega(d^{2c})}$ for any $d>(1/c)^{C/c}$ where $C$ is a sufficiently large constant}. Finally, let $S$ be the set of matrices $\{\vec A_i\}_{i=1}^{ |S'|/m}$ where for each $\vec B_i \in S''$ we consider the Singular Value Decomposition $\vec B_i = \vec U_i \vec \Sigma_i \vec V_i^\top$ and we let $\vec A_i$ be the matrix obtained by replacing the diagonal matrix $\vec \Sigma_i$ with identity (i.e., changing all singular values to 1). We will show that $S$ is the set of matrices satisfying the desideratum of \Cref{fact:setofmatrices-improved}.

    In particular, we claim the following. Let $C$ be a sufficiently large absolute constant, then:
    \begin{enumerate}[label = (\roman*)]
        \item For every $i \in |S''|$, all singular values of $\vec B_i$ belong in $[1-C m^2 d^{-1/2+c},1+C m^2 d^{-1/2+c}]$. \label{it:singular_vals}
        \item For every $i \in |S''|$, it holds $\| \bA_i - \bB_i \|_\fr \lesssim m^{2.5}d^{-1/2+c}$. \label{it:normbound}
        \item For every $i,j=1,\ldots,|S''|$, it holds $\|\bB_i\bB_j^\top \|_\op \lesssim m^2 d^{-1/2+c}$. \label{it:cor}
    \end{enumerate}
    Given the above, the proof of \Cref{fact:setofmatrices-improved} follows immediately by noting that
\begin{align*}
        \| \vec A_i \vec A_j^\top \|_\op  &= 
         \| (\vec B_i + \bA_i - \bB_i) \vec (\vec B_j + \bA_j - \bB_j)^\top \|_\op \\
         &\leq \| \bB_i \vec B_j^\top \|_\op + \| \vec B_i (\bA_j - \bB_j)^\top \|_\op + \| (\bA_i - \bB_i) \vec B_j^\top  \|_\op + \| (\bA_i - \bB_i)(\bA_j - \bB_j)^\top  \|_\op \\
         &\leq \| \bB_i \vec B_j^\top \|_\op  + \| \vec B_i \|_\op \|\bA_j - \bB_j\|_\fr +  \| \vec B_j^\top \|_\op  \|\bA_i - \bB_i\|_\fr + \| \bA_i - \bB_i\|_\fr  \| \bA_j - \bB_j\|_\fr \\
&\lesssim m^2 d^{-1/2+c} + m^{3}d^{-1/2+c} + m^{5}d^{-1/4+2c} \\
         &\lesssim  d^{-1/2+2c} \;,
    \end{align*}
    where the second line uses triangle inequality, the third line uses the sub-multiplicative property of the operator norm , i.e., that $\|\vec U \vec V\|_\op \leq \|\vec U \|_\op \|\vec V\|_\op$ as well as the fact $\|\vec V\|_\op \leq \|\vec V\|_\fr$, the fourth line uses our three claims (that we show later on) and the \new{last line uses our assumption $m\ll d^{c/5}$}.

    We now prove the three claims. For \Cref{it:singular_vals}, consider the matrix $\bB_i\bB_i^\top$ (which is a square $m \times m$ matrix). Using \Cref{fact:nearly-orth-vec}, the sum of the absolute values of its non-diagonal entries is
    \begin{align*}
        R = \sum_{k \neq \ell} | \vec u_{(i-1)m+k}^\top \vec u_{(i-1)m+\ell} | \lesssim m^2 d^{-1/2+c} \;.
    \end{align*}
    The diagonal entries of $\bB_i\bB_i^\top$ are all equal to one. Thus, by the Gershgorin's disc theorem \Cref{fact:Gershgorin}, every eigenvalue of $\bB_i\bB_i^\top$, i.e., singular value of $\bB_i$, lies  the interval $[1-R,1+R]$.

    For proving \Cref{it:normbound}, we note that 
    \begin{align*}
        \left\| \bA_i - \bB_i \right\|_\fr  = \sqrt{\sum_{k=1}^m(\sigma_k(\bB_i)-1)^2} 
        \leq \sqrt{m \cdot (R-1)^2}
        \lesssim m^{2.5}d^{-1/2+c} \;.
    \end{align*}
    
    Finally, regarding \Cref{it:cor}, for every $i,j \in[|S''|]$ with $i\neq j$, we have that
    \begin{align*}
        \|\bB_i\bB_j^\top \|_\op &\leq \sup_{\bz \in \cS^{m-1}} \bz^\top \bB_i\bB_j^\top \bz 
        \leq \sup_{\bz \in \cS^{m-1}} \left\langle \sum_{k \in [m]} z_k \vec u_{(i-1)m+k},  \sum_{\ell \in [m]} z_\ell \vec u_{(j-1)m+\ell}  \right\rangle \\
        &\leq  \sup_{\bz \in \cS^{m-1}} \sum_{k,\ell \in [m]}  z_k z_\ell  \left\langle  \vec u_{(i-1)m+k}, \vec u_{(j-1)m+\ell}   \right\rangle \\
        &\lesssim d^{-1/2+c} \sup_{z \in \cS^{m-1}} \sum_{k,\ell \in [m]}z_k z_\ell  
        \lesssim m^2 d^{-1/2+c} \;,
    \end{align*}
   where the last line uses \Cref{fact:nearly-orth-vec}.
\end{proof}

\section{Lower Bounds for Low-Degree Polynomial Tests}\label{sec:low_degree}

\begin{problem} \label{def:nongaussian_new}
	Let a distribution $A$ on $\R^m$. For a matrix $\bV \in \R^{m \times d}$, we let $P_{A,\bV}$ be the distribution as in \Cref{eq:hiden-dir}, 
i.e., the distribution that coincides with $A$ on the subspace spanned by the rows of $\bV$ and is standard Gaussian in the orthogonal subspace. Let $S$ be the set of nearly orthogonal vectors from \Cref{fact:setofmatrices}. Let $\cS = \{ P_{A,v} \}_{u \in S}$. We define the simple hypothesis testing problem where the null hypothesis is $\mathcal{N}(\vec 0,I_d)$ and the alternative hypothesis is $P_{A,\bV}$ for some $\bV$ uniformly selected from $S$.
\end{problem}

 We now describe the model in more detail. We will consider tests that are thresholded polynomials of low-degree, i.e., output $H_1$ if the value of the polynomial exceeds a threshold and $H_0$ otherwise. We need the following notation and definitions.
For a distribution $D$ over $\cX$, we use $D^{\otimes n}$ to denote the joint distribution of $n$ i.i.d.\ samples from $D$.
For two functions $f:\cX \to \R$, $g: \cX \to R$ and a distribution $D$, we use $\langle f, g\rangle_{D}$ to denote the inner product $\E_{X \sim D}[f(X)g(X)]$.
We use $ \|f\|_{D}$ to denote $\sqrt{\langle f, f \rangle_{D} }$.
We say that a polynomial $f(x_1,\dots,x_n):\R^{n \times d} \to \R$ has sample-wise degree $(r,\ell )$ if each monomial uses at most $\ell$ different samples from $x_1,\dots,x_n$ and uses degree at most $r$ for each of them.
Let $\cC_{r,\ell}$ be linear space of all polynomials of sample-wise degree $(r,\ell)$ with respect to the inner product defined above.
For a function $f:\R^{n \times d} \to \R$, we use $f^{\leq r, \ell}$ to be the orthogonal projection onto $\cC_{r,\ell}$   with respect to the inner product $\langle \cdot , \cdot \rangle_{D_0^{\otimes n}}$.  Finally, for the null distribution $D_0$ and a distribution $P$, define the likelihood ratio $\overline{P}^{\otimes n}(x) := {P^{\otimes n}(x)}/{D_0^{\otimes n}(x)}$.

\begin{definition}[$n$-sample $\tau$-distinguisher]
For the hypothesis testing problem between  $D_0$ (null distribution) and $D_1$ (alternate distribution) over $\cX$,
we say that a function $p : \cX^n \to \R$ is an $n$-sample $\tau$-distinguisher if $|\E_{X \sim  D_0^{\otimes n}}[p(X)] - \E_{X \sim D_1^{\otimes n}}[p(X)]| \geq \tau \sqrt{\Var_{X \sim D_0^{\otimes n}} [p(X)] }$. We call $\tau$ the \emph{advantage} of the polynomial $p$. 
\end{definition}
Note that if a function $p$ has advantage $\tau$, then the Chebyshev's inequality implies that one can furnish a test $p':\cX^n \to \{D_0,D_1\}$ by thresholding $p$ such that the probability of error under the null distribution is at most $O(1/\tau^2)$. 
We will think of the advantage $\tau$ as the proxy for the inverse of the probability of error (see  Theorem 4.3 in \cite{kunisky2022notes}  for a formalization of this intuition under certain assumptions) and we will show that the advantage of all polynomials up to a certain degree is $O(1)$. 
It can be shown that for hypothesis testing problems of the form of   \Cref{def:nongaussian_new}, 
 the best possible advantage among all polynomials in $\cC_{r,\ell}$ is captured by the low-degree likelihood ratio (see, e.g., ~\cite{BBHLS20,kunisky2022notes}):
\begin{align*}
    \left\| \E_{v \sim \cU(S)}\left[ \left( \overline{P}_{A,\bV}^{\otimes n}  \right)^{\leq r, \ell } \right]  - 1  \right\|_{D_0^{\otimes n}},
\end{align*}
where in our case $D_0 = \cN(\vec 0,\bI_d)$.

To show that the low-degree likelihood ratio is small, we use the result from \cite{BBHLS20} stating that a lower bound for the SQ dimension translates to an upper bound for the low-degree likelihood ratio. Therefore, given that we have already established in  previous section that $\mathrm{SD}(\cB(\{P_{A,\bV} \}_{\bV \in S},\cN(\vec 0,\bI_d)), \gamma,\beta)=2^{d^c}$ for $\gamma=\Omega(d)^{(t+1)/10}\chi^2(A,\cN(\vec 0,\bI_d))$ and $\beta= \chi^2(A,\cN(0,1))$, we one can obtain the corollary:

        \begin{theorem}\label{cor:low-deg-hardness-general-problem} 
        Let a sufficiently small positive constant $c$. Let the hypothesis testing problem of  \Cref{def:nongaussian_new}  the distribution $A$ matches the first $t$  moments with $\cN(\vec 0, \bI_m)$. For any $d \in \Z_+$ with $d = t^{\Omega(1/c)}$, any $n \leq \Omega(d)^{(t+1)/10}/\chi^2(A,\normal(\vec 0,\vec I_m))$ and any even integer $\ell < d^{c}$, we have that
        \begin{align*}
            \left\| \E_{v \sim \cU(S)}\left[ \left( \overline{P}_{A,\bV}^{\otimes n}  \right)^{\leq \infty, \ell } \right]  - 1  \right\|_{D_0^{\otimes n}} \leq 1\;.
        \end{align*}
    \end{theorem}

The interpretation of this result is that unless the number of samples used $n$ is greater than $\Omega(d)^{(t+1)/10}/\chi^2(A,\normal(\vec 0,\vec I_m))$, any polynomial of degree roughly up to $d^{c}$  fails to be a good test (note that any polynomial of degree $\ell$ has sample-wise degree at most $(\ell,\ell)$).

\end{document}